\documentclass[twoside,11pt]{article}
\usepackage{jair, theapa, rawfonts}

\jairheading{63}{2018}{461-494}{4/18}{11/18}
\ShortHeadings{Proximal Gradient TD Algorithms}{Liu, Gemp, Ghavamzadeh, Liu, Mahadevan, \& Petrik}
\firstpageno{461}
\usepackage{times}
\usepackage{epsfig}
\usepackage{algorithm}
\usepackage{algorithmic}
\usepackage{subfig}
\usepackage{color}
\usepackage{amsmath, amsfonts, amssymb, amsthm}
\usepackage{mathtools}
\usepackage{mathrsfs}
\mathtoolsset{showonlyrefs}
\usepackage{tabularx}
\newcounter{thm_counter}
\newcounter{lem_counter}
\newcounter{pro_counter}
\newcounter{ass_counter}

\newtheorem{theorem}[thm_counter]{Theorem}
\newtheorem{proposition}[pro_counter]{Proposition}
\newtheorem{lemma}[lem_counter]{Lemma}
\newtheorem{assumption}[ass_counter]{Assumption}
\newtheorem{definition}[ass_counter]{Definition}

\newcommand{\rc}{\color{black}}

\newcommand{\negt}[1]{[#1]_-}
\newcommand{\pos}[1]{[#1]_+}
\usepackage[inline,ignoremode]{trackchanges}

\addeditor{Place_Holder} 
\addeditor{mgh}

\addeditor{lb}

\title{{\rc Proximal Gradient Temporal Difference Learning:\\ Stable Reinforcement Learning with Polynomial Sample Complexity}}

\author{\name Bo Liu \email boliu@auburn.edu \\
      \addr Auburn University\\
      Auburn, AL 36849, USA
      \AND
      \name Ian Gemp \email imgemp@cs.umass.edu \\
      \addr UMass Amherst\\
      Amherst, MA 01002, USA
      \AND
      \name Mohammad Ghavamzadeh \email mgh@fb.com \\
      \addr Facebook AI Research \\ Menlo Park, CA 94025, USA
      \AND 
      \name Ji Liu \email jliu@cs.rochester.edu \\
      \addr University of Rochester\\
      Rochester, NY 14627, USA
      \AND 
      \name Sridhar Mahadevan \email mahadeva@cs.umass.edu \\
      \addr UMass Amherst\\
      Amherst, MA 01002, USA
      \AND
      \name Marek Petrik \email mpetrik@cs.unh.edu\\
      \addr University of New Hampshire \\ Durham, NH 03824, USA
}

\begin{document}
\maketitle

\begin{abstract} 
In this paper, we introduce proximal gradient temporal difference learning, which provides a principled way of designing and analyzing true stochastic gradient temporal difference learning algorithms. We show how gradient TD (GTD) reinforcement learning methods can be formally derived,  not by starting from their original objective functions, as previously attempted, but rather from a primal-dual saddle-point objective function. We also conduct a saddle-point error analysis to obtain finite-sample bounds on their performance. Previous analyses of this class of algorithms use stochastic approximation techniques to prove asymptotic convergence, and do not provide any finite-sample analysis. We also propose an accelerated algorithm, called GTD2-MP, that uses proximal ``mirror maps'' to yield an improved convergence rate. The results of our theoretical analysis imply that the GTD family of algorithms are comparable and may indeed be preferred over existing least squares TD methods for off-policy learning, due to their linear complexity. We provide experimental results showing the improved performance of our accelerated gradient TD methods.
\end{abstract} 


\section{Introduction}

Obtaining a true stochastic gradient temporal difference method has been a longstanding goal of reinforcement learning (RL) for almost three decades~\cite{ndp:book,sutton-barto:book} ever since it was discovered that the original TD method was unstable in many off-policy applications, such as Q-learning, where the target behavior being learned and the exploratory behavior producing samples differ. \citeauthor{Sutton:GTD1:2008}~\citeyear{Sutton:GTD1:2008} and~\citeauthor{tdc:2009}~\citeyear{tdc:2009} proposed a family of gradient-based temporal difference (GTD) algorithms, which yielded several interesting properties. A key property of this class of GTD algorithms is that they are asymptotically convergent in the off-policy setting. However, the original derivation of these methods was somewhat ad-hoc, as the derivation from the original loss functions involved some non-mathematical steps (such as an arbitrary decomposition of the resulting product of gradient terms). Consequently, the resulting convergence analysis was also weakened, and limited to showing the asymptotic convergence using stochastic approximation~\cite{borkar:book}. Despite these shortcomings, gradient TD was a significant advance,  as previous work on off-policy methods, such as TD($\lambda$), do not have convergence guarantees in the off-policy setting. A further appealing property of these algorithms is the first-order computational complexity that allows them to scale more gracefully to high-dimensional problems, unlike the widely used least-squares TD (LSTD) approaches~\cite{bradtke-barto:LSTD} that only perform well with moderate size reinforcement learning (RL) problems, due to their quadratic (w.r.t.~the dimension of the feature space) computational cost per iteration. 

Unfortunately, despite the nomenclature, GTD algorithms are {\em not true stochastic gradient methods with respect to their original objective functions}, as pointed out by~\citeauthor{szepesvari2010algorithms}~\citeyear{szepesvari2010algorithms}. The reason is not surprising: the gradient of the objective function involves products of terms, which cannot be sampled directly. Consequently,  their original derivation involved a rather ad-hoc splitting of terms, which was justified more or less from intuition. In this paper, we take a major step forward in resolving this problem by showing a principled way of designing true stochastic gradient TD algorithms by using a primal-dual saddle point objective function, derived from the original objective functions, coupled with the powerful machinery of {\em operator splitting}~\cite{BOOK2011PROXSPLIT}. A significant advantage of our approach is that it enables undertaking a precise {\em finite sample analysis} of convergence, which provides deeper insight into the actual running times of the GTD methods beyond the standard asymptotic analysis. 

Since in real-world applications of RL, we have access to only a finite amount of data, finite-sample analysis of gradient TD algorithms is essential as it clearly shows the effect of the number of samples (and the parameters that play a role in the sampling budget of the algorithm) {on} their final performance. However, most of the work on the finite-sample analysis in RL has been focused on  batch RL (or approximate dynamic programming) algorithms 
 (e.g.,~\citeR{Kakade02AO,Munos08FT,antos08learning,Lazaric10AC}), especially those that are  least squares TD (LSTD)-based (e.g.,~\citeR{Lazaric_finite-sampleanalysis,lstdrp:nips2010,LASSOTD:2011,Lazaric12FS}), and more importantly restricted to  the on-policy setting. In this paper, we provide the finite-sample analysis of the GTD family of algorithms, a relatively novel class of gradient-based TD methods that are guaranteed to converge even in the off-policy setting, and for which, to the best of our knowledge, no finite-sample analysis has been reported. This analysis is challenging because {\bf 1)} the stochastic approximation methods that have been used to prove the asymptotic convergence of these algorithms do not address convergence rate analysis;  {\bf 2)} as we explain  in detail in Section~\ref{subsec:GTD-algos},  the techniques used for the analysis of the stochastic gradient methods cannot be applied here;  {\bf 3)} finally, the difficulty of finite-sample analysis in the off-policy setting. It should also be noted that there exists very little literature on the finite-sample analysis in the off-policy setting, even for the LSTD-based algorithms that have been extensively studied.


{
The major contributions of this paper include 
\begin{itemize}
\item The first finite-sample analyses of the TD algorithms with linear computational complexity, which is also one of the first few finite-sample {analyses} of off-policy convergent TD algorithms.
\item A novel framework for designing gradient-based TD algorithms with Bellman Error based objective functions, as well as the design and analysis of several improved GTD methods that result from our novel approach of formulating gradient TD methods as true stochastic gradient algorithms w.r.t.~a saddle-point objective function.  
\end{itemize}

We then use the techniques applied in the analysis of the stochastic gradient methods to propose a unified finite-sample analysis for the previously proposed GTD algorithms as well as our novel gradient TD algorithms. Finally, given the results of our analysis, we study the GTD class of algorithms from several different perspectives, including acceleration in convergence and learning with biased importance sampling factors.

}


\section{Preliminaries}
\label{sec:preliminaries}

Reinforcement Learning (RL)~\cite{ndp:book,sutton-barto:book} is a subfield of machine learning, studying a class of problems in which an agent interacts with an unfamiliar, dynamic and stochastic environment, with the goal of optimizing some measure of its long-term performance. This interaction is conventionally modeled as a Markov decision process (MDP). An MDP is defined as the tuple $({\mathcal{S},\mathcal{A},P_{ss'}^{a},R,\gamma})$, where $\mathcal{S}$ and $\mathcal{A}$ are the sets of states and actions, $P_{ss'}^{a}$ is the transition kernel specifying the probability of transition from state $s\in\mathcal{S}$ to state $s'\in\mathcal{S}$ by taking action $a\in\mathcal{A}$, $R(s,a):\mathcal{S}\times\mathcal{A}\to\mathbb{R}$ is the reward function bounded by $R_{\max}$, and $0\leq\gamma<1$ is a discount factor. A stationary policy $\pi:\mathcal{S}\times\mathcal{A}\to\left[{0,1}\right]$ is a probabilistic mapping from states to actions. The main objective of an RL algorithm is to find an optimal policy. In order to achieve this goal, a key step in many algorithms is to calculate the value function of a given policy $\pi$, i.e.,~$V^{\pi}:\mathcal{S}\to\mathbb{R}$, a process known as {\em policy evaluation}. It is known that $V^\pi$ is the unique fixed-point of the {\em Bellman operator} $T^\pi$, i.e.,
\begin{equation}
\label{eq:BellmanEq}
V^\pi = T^\pi V^\pi = R^\pi + \gamma P^\pi V^\pi,
\end{equation}
where $R^\pi$ and $P^\pi$ are the reward function and transition kernel of the Markov chain induced by policy $\pi$. In Eq.~\eqref{eq:BellmanEq}, we may imagine $V^\pi$ as {an} $|\mathcal{S}|$-dimensional vector and write everything in vector/matrix form. In the following, to simplify the notation, we often drop the dependence of $T^\pi$, $V^\pi$, $R^\pi$, and $P^\pi$ {on} $\pi$. 

Off-policy learning refers to learning about one way of behaving, termed as the \textit{target policy}, from data generated by another way of selecting actions, termed as the \textit{behavior policy}. The target policy is often a deterministic policy that approximates the optimal policy. On the other hand, 
the behavior policy is often stochastic, exploring all possible actions in each state in order to find the optimal policy. There are several benefits of learning with behavior policies. First, it allows freeing the behavior policy from the target policy and thus has a greater variety of exploration strategies to be used. Secondly, it enables learning from training data generated by unrelated controllers, including manual human control, and from previously collected data. The third reason for interest in off-policy learning is that it permits learning about multiple target policies (e.g., optimal policies for multiple sub-goals) from a single stream of data generated by a single behavior policy, i.e., parallel learning is allowed for off-policy learning~\cite{maei2011thesis}.

In the paper, we denote by $\pi_b$, the behavior policy that generates the data, and by $\pi$, the target policy that we would like to evaluate. They are the same in the on-policy setting and different in the off-policy {setting}. For each state-action pair $(s_i,a_i)$, such that $\pi_b(a_i|s_i)>0$, we define the importance-weighting factor $\rho_i = \pi(a_i|s_i)/\pi _b(a_i|s_i)$ with $\rho_{\max}\geq 0$ being its maximum value over the state-action pairs.

When $\mathcal{S}$ is large or infinite, we often use a linear approximation architecture for $V^\pi$ with parameters $\theta\in\mathbb{R}^d$ and $L$-bounded basis functions $\{\varphi_i\}_{i=1}^d$, i.e.,~$\varphi_i:\mathcal{S}\rightarrow\mathbb{R}$ and $\max_i||\varphi_i||_\infty\leq L$. We denote by $\phi(\cdot)=\big(\varphi_1(\cdot),\ldots,\varphi_d(\cdot)\big)^\top$ the feature vector and by $\mathcal{F}$ the linear function space spanned by the basis functions $\{\varphi_i\}_{i=1}^d$, i.e.,~$\mathcal{F}=\big\{f_\theta\mid\theta\in\mathbb{R}^d\;\text{and}\;f_\theta(\cdot)=\phi(\cdot)^\top\theta\big\}$. We may write the approximation of $V$ in $\mathcal{F}$ in the vector form as $\hat{v}=\Phi\theta$, where $\Phi$ is the $|\mathcal{S}|\times d$ feature matrix. When only $n$ training samples of the form $\mathcal{D}=\big\{\big(s_i,a_i,r_i=r(s_i,a_i),s'_i\big)\big\}_{i=1}^n,\;s_i\sim\xi,\;a_i\sim\pi_b(\cdot|s_i),\;s'_i\sim P(\cdot|s_i,a_i)$, are available ($\xi$ is a distribution over the state space $\mathcal{S}$), we may write the {\em empirical Bellman operator} $\hat{T}$ for a function in $\mathcal{F}$ as 
\begin{equation}
\label{eq:EmpBellmanEq}
\hat{T}(\hat \Phi \theta ) = \hat R + \gamma\hat\Phi '\theta,
\end{equation}
where $\hat{\Phi}$ (resp.~$\hat{\Phi}'$) is the empirical feature matrix of size $n\times d$, whose $i$-th row is the feature vector $\phi(s_i)^\top$ (resp.~$\phi(s'_i)^\top$), and $\hat{R}\in\mathbb{R}^n$ is the reward vector, whose $i$-th element is $r_i$. $\phi(s_i)$ (resp.~$\phi(s'_i)$) will be denoted as $\phi$ (resp.~$\phi'$) for short. 
We denote by $\delta_i(\theta)=r_i+\gamma\phi_i^{'\top}\theta-\phi_i^\top\theta$, the TD error for the $i$-th sample $(s_i,r_i,s'_i)$ and define $\Delta\phi_i=\phi_i-\gamma\phi'_i$. Finally, we define the matrices $A$ and $C$, and the vector $b$ as
\begin{equation}
A := \mathbb{E}\big[\rho_i\phi_i(\Delta\phi_i)^\top\big],\;\;b := \mathbb{E}\left[\rho_i\phi_ir_i\right],\;\;C := \mathbb{E}[\phi_i\phi_i^\top],
\label{eq:abc}   
\end{equation}
\noindent where the expectations are w.r.t.~$\xi$ and $P^{\pi_b}$. We also denote by $\Xi$, the diagonal matrix whose elements are $\xi(s)$, and ${\xi _{\max }} := {\max _s}\xi (s)$. For each sample $i$ in the training set $\mathcal{D}$, we can calculate an unbiased estimate of $A$, $b$, and $C$ as follows:
\begin{equation}
\hat{A}_i := \rho_i\phi_i\Delta\phi_i^\top, \quad\; \hat{b}_i := \rho_ir_i\phi_i, \quad\; \hat{C}_i := \phi_i\phi_i^\top.
\label{eq:atbtct}
\end{equation}


\subsection{Gradient-based TD Algorithms}
\label{subsec:GTD-algos}

The class of gradient-based TD (GTD) algorithms {was} proposed by~\citeauthor{Sutton:GTD1:2008,tdc:2009}~\citeyear{Sutton:GTD1:2008,tdc:2009}. These algorithms target two objective functions: the {\em norm of the expected TD update} (NEU) and the {\em mean-square projected Bellman error} (MSPBE), defined as (see e.g.,~\citeR{maei2011thesis})\footnote{It is important to note that $T$ in Eq.~\eqref{eq:neu} and Eq.~\eqref{eq:mspbe} is $T^\pi$, the Bellman operator of the target policy $\pi$.}
\begin{align}
{\rm {NEU}}(\theta) &= ||\Phi^\top\Xi(T\hat{v}-\hat{v})||^{2}\;,
\label{eq:neu}\\
{\rm {MSPBE}}(\theta) 
&=  ||\hat{v} - \Pi T\hat{v}||_{\xi}^2 = ||\Phi^\top\Xi(T\hat{v}-\hat{v})||_{C^{-1}}^2\;,
\label{eq:mspbe}
\end{align}
%
\noindent where $C=\mathbb{E}[\phi_i\phi_i^\top]=\Phi^\top\Xi\Phi$ is the covariance matrix defined in Eq.~\eqref{eq:abc} and is assumed to be non-singular, and for $\forall x \in {\mathbb{R}^{d \times 1}}, ||x||_{{C^{ - 1}}}^2\; = {x^ \top }{C^{ - 1}}x$. $\Pi = \Phi(\Phi ^\top\Xi\Phi)^{-1}\Phi^\top\Xi$ is the orthogonal projection operator {onto} the function space $\mathcal{F}$, i.e.,~for any bounded function $g$, $\Pi g = \arg {\min _{f \in {\cal F}}}||g - f|{|_\xi } = \arg {\min _{f \in {\cal F}}}{(g - f)^ \top }{\rm{diag}}(\xi )(g - f)$. From Eqs.~\eqref{eq:neu} and~\eqref{eq:mspbe}, it is clear that NEU and MSPBE are square unweighted and weighted by $C^{-1}$, $\ell_2$-norms of the quantity $\Phi^\top\Xi(T\hat{v}-\hat{v})$, respectively, and thus, the two objective functions can be unified as
\begin{equation}
J(\theta)=||\Phi^\top\Xi(T\hat{v}-\hat{v})||_{M^{-1}}^{2} = ||\mathbb{E}[\rho_i\delta_i(\theta)\phi_i]||_{M^{-1}}^{2},
\label{eq:j}
\end{equation}
with $M$ equal to the identity matrix $I$ for NEU and to the covariance matrix $C$ for MSPBE. The second equality in Eq.~\eqref{eq:j} holds because of the following lemma from Section 4.2 in the work of~\citeauthor{maei2011thesis}~\citeyear{maei2011thesis}.
\begin{lemma}
\label{lem:e}
(Importance-weighting for off-policy TD) 
Let $\mathcal{D}=\big\{\big(s_i,a_i,r_i,s'_i\big)\big\}_{i=1}^n,\;s_i\sim\xi,\;a_i\sim\pi_b(\cdot|s_i),\;s'_i\sim P(\cdot|s_i,a_i)$ be a training set generated by the behavior policy $\pi_b$ and $T$ be the Bellman operator of the target policy $\pi$. Then, we have 
\begin{equation}
\Phi^\top\Xi (T\hat v - \hat v) = \mathbb{E}\big[\rho_i\delta_i(\theta)\phi_i\big] = b-A\theta.
\end{equation}
\end{lemma}
{
\begin{proof}
We give a proof sketch here. {Refer to Section 4.2 in the work of~\citeauthor{maei2011thesis}~\citeyear{maei2011thesis} for a detailed proof}.
\begin{align}
\nonumber
&{\Phi ^\top }\Xi (T\hat v - \hat v) \\
\nonumber
&= \sum\limits_{s,a,s'} {\xi (s)\pi (a|s)} P(s'|s,a)\delta (\theta |s,a,s')\phi (s)\\
\nonumber
&= \sum\limits_{s,a,s'} {\xi (s)\frac{{\pi (a|s)}}{{{\pi _b}(a|s)}}{\pi _b}(a|s)} P(s'|s,a)\delta (\theta |s,a,s')\phi (s)\\
 \nonumber
&= \sum\limits_{s,a,s'} {\xi (s)\rho (s,a){\pi _b}(a|s)} P(s'|s,a)\delta (\theta |s,a,s')\phi (s)\\
  \nonumber
&= \mathbb{E}[{\rho _t}{\delta _t}(\theta ){\phi _t}]
 \\
 \nonumber
&= b-A\theta
\end{align}
\end{proof}
}
Motivated by minimizing the NEU and MSPBE objective functions using the stochastic gradient methods, the GTD and GTD2 algorithms were proposed with the following update rules: 
\begin{align}
\label{eq:gtd}
\hspace{-1.75cm}\textbf{GTD:}\quad\quad y_{t + 1} &= y_t + \alpha_t\big(\rho_t\delta_t(\theta_t)\phi_t - y_t\big), \\
\theta_{t + 1} &= \theta_t + \alpha_t\rho_t\Delta\phi_t(y_t^\top\phi_t), \nonumber
\end{align}
\begin{align}
\label{eq:gtd2}
\hspace{-1cm}\textbf{GTD2:}\quad\quad y_{t + 1} &= y_t + \alpha_t\big(\rho_t\delta_t(\theta_t) - \phi_t^\top y_t\big)\phi_t,\\
\theta_{t + 1} &= \theta_t + \alpha_t\rho_t\Delta\phi_t(y_t^\top\phi_t). \nonumber
\end{align}
However, it has been shown that the above update rules do not update the value function parameter $\theta$ in the gradient direction of NEU and MSPBE, and thus, NEU and MSPBE are not the true objective functions of the GTD and GTD2 algorithms~\cite{szepesvari2010algorithms}. Consider the NEU objective function in Eq.~\eqref{eq:neu}. Taking its gradient w.r.t.~$\theta$, we obtain
\begin{eqnarray}
\nonumber
 - \frac{1}{2}\nabla {\rm{NEU}}(\theta ) &=&  - \big(\nabla\mathbb{E}\big[\rho_i\delta_i(\theta)\phi^\top_i\big]\big)\mathbb{E}\big[\rho_i\delta_i(\theta)\phi_i\big] \\
 \nonumber
 &=&  - \big(\mathbb{E}\big[\rho_i\nabla\delta_i(\theta)\phi^\top_i\big]\big)\mathbb{E}\big[\rho_i\delta_i(\theta)\phi_i\big]\\
 &=& \mathbb{E}\big[\rho_i\Delta\phi_i\phi_i^\top\big]\mathbb{E}\big[\rho_i\delta_i(\theta)\phi_i\big].
 \label{eq:neu-grad}
\end{eqnarray}
If the gradient can be written as a single expectation, then it is straightforward to use a stochastic gradient method. However, we have a product of two expectations in Eq.~\eqref{eq:neu-grad}, and unfortunately, due to the correlation between them, the sample product (with a single sample) won't be an unbiased estimate of the gradient. To tackle this, the GTD algorithm uses an auxiliary variable $y_t$ to estimate $\mathbb{E}\big[\rho_i\delta_i(\theta)\phi_i\big]$, and thus, the overall algorithm is no longer a true stochastic gradient method w.r.t.~NEU. It can be easily shown that the same problem exists for GTD2 w.r.t.~the MSPBE objective function. This prevents us from using the standard convergence analysis techniques of stochastic gradient descent methods to obtain a finite-sample performance bound for the GTD and GTD2 algorithms.

It should also be noted that in the original publications of GTD/GTD2 algorithms~\cite{Sutton:GTD1:2008,tdc:2009}, the authors discussed handling the off-policy scenario using both importance and rejection sampling. In rejection sampling, a sample $({s_i},{a_i},{r_i},s'_i)$ is rejected and the parameter $\theta$ {is not updated} if $\pi(a_i|s_i) = 0$. This sampling strategy is not efficient since a lot of samples will be discarded if $\pi_b$ and $\pi$ are very different. 


\subsection{Related Work}

Before we present a finite-sample performance bound for GTD and GTD2, it would be helpful to give a brief overview of the existing literature on the finite-sample analysis of the TD algorithms. The convergence rate of the TD algorithms mainly depends on $(d,n,\nu)$, where $d$ is the size of the approximation space (the dimension of the feature vector), $n$ is the number of samples, and $\nu$ is the smallest eigenvalue of the sample-based covariance matrix $\hat C=\hat\Phi^\top\hat\Phi$, i.e.,~$\nu=\lambda_{\min }(\hat C)$.

\citeauthor{antos08learning}~\citeyear{antos08learning} proved an error bound of $O(\frac{d\log d}{n^{1/4}})$ for LSTD in bounded spaces. \citeauthor{Lazaric_finite-sampleanalysis}~\citeyear{Lazaric_finite-sampleanalysis} proposed an LSTD analysis in linear spaces and obtained a tighter bound of $O(\sqrt{\frac{d\log d}{n\nu })}$ and later used it to derive a bound for the least-squares policy iteration (LSPI) algorithm~\cite{Lazaric12FS}. \citeauthor{bruno:lstdlambda}~\citeyear{bruno:lstdlambda} recently proposed the first convergence analysis for LSTD$(\lambda)$ and derived a bound of $\tilde O(d/\nu\sqrt n )$. The analysis differs slightly from the work of~\citeauthor{Lazaric_finite-sampleanalysis}~\citeyear{Lazaric_finite-sampleanalysis} and the bound is weaker in terms of $d$ and $\nu$. Another recent result is by~\citeauthor{drift:prashanth2014fast}~\citeyear{drift:prashanth2014fast} that {uses} stochastic approximation to solve LSTD$(0)$, where the resulting algorithm is exactly TD$(0)$ with random sampling (samples are drawn i.i.d.~and not from a trajectory), and report a Markov design bound (the bound is computed only at the states used by the algorithm) of $O(\sqrt{\frac{d}{n\nu}})$ for LSTD$(0)$. All these results are for the on-policy setting, except the result by~\citeauthor{antos08learning}~\citeyear{antos08learning} that also holds for the off-policy formulation. Another {result} in the off-policy setting is by~\citeauthor{pires:2012:inverse}~\citeyear{pires:2012:inverse} that uses a bounding trick and improves the result of~\citeauthor{antos08learning}~\citeyear{antos08learning} by a $\log d$ factor. 
Another line of work is by~\citeauthor{yu2012error}~\citeyear{yu2012error}, which provides error bounds of LSTD algorithms for a wide range of problems including the scenario that $||A||_{\xi}$ is unbounded, which is beyond the scope of the aforementioned literature and our paper.

The line of research reported here has much in common with work on proximal reinforcement learning~\cite{proximalrl}, which explores first-order reinforcement learning algorithms using {\em mirror maps}~\cite{bubeck2014optml,juditsky2008solving} to construct primal-dual spaces. This work began originally with a dual space formulation of first-order sparse TD learning~\cite{mahadevan:MID:2012}. The saddle point formulation for off-policy TD learning was initially explored by~\citeauthor{ROTD:NIPS2012}~\citeyear{ROTD:NIPS2012}, where the objective function is the norm of the approximation residual of a linear inverse problem \cite{pires:2012:inverse}. A sparse off-policy GTD2 algorithm with regularized dual averaging is introduced by~\citeauthor{ZHIWEI2014}~\citeyear{ZHIWEI2014}. These studies provide different approaches to formulating the problem {1)} as a variational inequality problem~\cite{juditsky2008solving,proximalrl}, {2)} as a linear inverse problem~\cite{ROTD:NIPS2012}, or {3)} as a quadratic objective function (MSPBE) using two-time-scale solvers~\cite{ZHIWEI2014}. In this paper, we are going to explore the true nature of the GTD algorithms as stochastic gradient {algorithms} w.r.t the convex-concave saddle-point formulations of NEU and MSPBE.


\section{Saddle-Point Formulation of GTD Algorithms}
\label{sec:saddle-point}

In this section, we show how the GTD and GTD2 algorithms can be formulated as true stochastic gradient (SG) algorithms by writing their respective objective functions, NEU and MSPBE, in the form of a convex-concave saddle-point. As discussed earlier, this new formulation of GTD and GTD2 as true SG methods allows us to use the convergence analysis techniques for SGs in order to derive finite-sample performance bounds for these RL algorithms. Moreover, it allows us to use more efficient algorithms that have been recently developed to solve SG problems, such as {\em stochastic Mirror-Prox} (SMP)~\cite{juditsky2008solving}, to derive more efficient versions of GTD and GTD2. 

A particular type of convex-concave saddle-point formulation is formally defined as  
\begin{equation}
\mathop {\min }\limits_\theta  \mathop {\max }\limits_y \big( {L(\theta ,y) = \left\langle {b-A\theta,y} \right\rangle  + F(\theta ) - K(y)} \big),
\label{eq:spgeneral}
\end{equation}
where $F(\theta)$ is a convex function and $K(y)$ is a smooth convex function such that  
\begin{equation}
\label{eq:saddle-gcondition}
K(y) - K(x) - \left\langle {\nabla K(x),y - x} \right\rangle  \le \frac{{{L_K}}}{2}||x - y|{|^2}.
\end{equation}
Next we follow the approach of prior work~\cite{juditsky2008solving,RobustSA:2009,chen2013optimal} and define the following error function for the saddle-point problem~\eqref{eq:spgeneral}.
\begin{definition}
The error function of the saddle-point problem~\eqref{eq:spgeneral} at each point $(\theta',y')$ is defined as 
\begin{equation}
{\rm{Err}}(\theta',y') = \max_y\;L(\theta',y) - \min_\theta\;L(\theta,y').
\label{eq:errdef}
\end{equation}
\end{definition}
In this paper, we consider the saddle-point problem~\eqref{eq:spgeneral} with $F(\theta)=0$ and $K(y)=\frac{1}{2}||y||_M^2$, i.e.,
\begin{equation}
\mathop{\min}\limits_{\theta}\mathop{\max}\limits_{y}\Big({L(\theta,y)=\left\langle {b-A\theta,y}\right\rangle -\frac{1}{2}||y||_{M}^{2}}\Big),
\label{eq:sp}
\end{equation}
where $A$ and $b$ were defined by Eq.~\eqref{eq:abc}, and $M$ is a positive definite matrix. It can be shown that $K(y)=\frac{1}{2}||y||^2_M$ satisfies the condition in Eq.~\eqref{eq:saddle-gcondition} by using the Taylor expansion of $y$ about $x$, i.e.,
\begin{equation}
\frac{1}{2}||y||_M^2 \ge \frac{1}{2}||x||_M^2 + M(y - x) + \frac{{{L_K}}}{2}||y - x||_M^2
\end{equation} 
We first show in Proposition~\ref{pro:1} that if $(\theta^*,y^*)$ is the saddle-point of problem~\eqref{eq:sp}, then $\theta^*$ will be the optimum of NEU and MSPBE defined in Eq.~\eqref{eq:j}. We then prove in Proposition~\ref{pro:2} that GTD and GTD2 in fact find this saddle-point.
%
\begin{proposition}
For any fixed $\theta$, we have $\frac{1}{2}J(\theta)=\mathop{\max}_yL(\theta,y)$, where $J(\theta)$ is defined by Eq.~\eqref{eq:j}.
\label{pro:1}
\end{proposition}
Readers familiar with Legendre-Fenchel duality or Legendre transform can easily prove this by using the fact that the Legendre-Fenchel convex conjugate function~\cite{boyd} of $f=\frac{1}{2}||Ax-b||{_{{M^{-1}}}}$ is ${f^{*}}=\frac{1}{2}||Ax-b|{|_{M}}$, and 
\begin{align}
f(x) = \frac{1}{2}||Ax - b|{|^2_{M^{ - 1}}} = {f^{**}}(x) = {\max _y}\big({y^\top}(Ax - b) - \frac{1}{2}||y||^2_M\big)
\end{align}
The second equality holds because $f(x)$ is convex.
We can also prove this another way as follows.
\begin{proof}
$L(\theta,y)$ is an unconstrained quadratic program w.r.t.~$y$, therefore, the optimal $y^*(\theta)=\arg\max_y L(\theta,y)$ can be analytically computed as
\begin{equation}
y^{*}(\theta)=M^{-1}(b-A\theta).
\label{eq:ystar}
\end{equation}
The result follows by plugging $y^*$ into Eq.~\eqref{eq:sp} and using the definition of $J(\theta)$ in Eq.~\eqref{eq:j} and Lemma~\ref{lem:e}.
\end{proof}
%
\begin{proposition}
\label{pro:2}
GTD and GTD2 are true stochastic gradient algorithms w.r.t.~the objective function $L(\theta,y)$ of the saddle-point problem~\eqref{eq:sp} with $M=I$ and $M=C={\Phi^{\top}}\Xi\Phi$ (the covariance matrix), respectively.
\end{proposition}
\begin{proof} 
It is easy to see that the gradient updates of the saddle-point problem~\eqref{eq:sp} (ascending in $y$ and descending in $\theta$) may be written as 
\begin{eqnarray}
\label{eq:sg}
{y_{t + 1}} &=& {y_t} + {\alpha _t}\left( b-{{A}{\theta _t}  - {M}{y_t}} \right),\,\\
\nonumber
{\theta _{t + 1}} &=& {\theta _t} + {\alpha _t}A^\top {y_t}.
\end{eqnarray}
We denote ${\hat M} := I$ (resp. ${\hat M} := {\hat C}$) for GTD (resp. GTD2). 
We may obtain the update rules of GTD and GTD2 by replacing $A$, $b$, and $C$ in Eq.~\eqref{eq:sg} with their unbiased estimates $\hat A$, $\hat b$, and $\hat C$ from Eq.~\eqref{eq:atbtct}, which completes the proof. 
\end{proof}


\section{Finite-Sample Analysis}
\label{sec:analysis}

In this section, we provide a finite-sample analysis for a revised version of the GTD/GTD2 algorithms. We first describe the revised GTD algorithms in Section~\ref{subsec:revised} and then dedicate the rest of Section~\ref{sec:analysis} to their sample analysis. Note that from now on we use the $M$ matrix (and its unbiased estimate $\hat M_t$) to have a unified analysis of GTD and GTD2 algorithms. As described earlier, $M$ is replaced by the identity matrix $I$ in GTD and by the covariance matrix $C$ (and its unbiased estimate $\hat C_t$) in GTD2.


\subsection{The Revised GTD Algorithms}
\label{subsec:revised}

The revised GTD algorithms that we analyze in this paper (see Algorithm~\ref{alg:pgtd2}) have three differences with the standard GTD algorithms of Eqs.~\eqref{eq:gtd} and~\eqref{eq:gtd2} (and Eq.~\eqref{eq:sg}). 
\begin{itemize}
\item We guarantee that the parameters $\theta$ and $y$ remain bounded by projecting them onto bounded convex feasible sets $\Theta$ and $Y$ defined in Assumption~\ref{ass:xyfeasible}. In Algorithm~\ref{alg:pgtd2}, we denote by $\Pi_\Theta$ and $\Pi_Y$, the projection {onto} sets $\Theta$ and $Y$, respectively. This is standard in stochastic approximation algorithms and has been used in off-policy TD($\lambda$)~\cite{yu2012error} and actor-critic algorithms (e.g.,~\citeR{Bhatnagar09NA}).
\item After $n$ iterations ($n$ is the number of training samples in $\mathcal{D}$), the algorithms return the weighted (by the step size) average of the parameters at all the $n$ iterations (see Eq.~\eqref{eq:bartheta}).
\item The step-size $\alpha_t$ is selected as described in the proof of Proposition~\ref{pro:hp} in the Appendix. Note that this fixed step size of $O(1/\sqrt{n})$ is required for the high-probability bound in Proposition~\ref{pro:hp} (see~\citeR{RobustSA:2009} for more details).
\end{itemize}

\begin{algorithm}
\caption{Revised GTD Algorithms}
\label{alg:pgtd2} 
\begin{algorithmic}[1]
\FOR {$t=1,\ldots,n$}
\STATE Update parameters

\begin{align}
\label{eq:pgtd2}
y_{t+1} &= \Pi_Y \Big(y_t + \alpha_t(\hat{b_t} - \hat{A}_t\theta_t - \hat{M}_ty_t)\Big) \nonumber \\
\theta_{t+1} &= \Pi_\Theta\Big(\theta_t + \alpha_t \hat{A}_t^\top y_t\Big)
\end{align}

\ENDFOR
\STATE OUTPUT

\begin{equation}
{\bar \theta _n} := \frac{{\sum\nolimits_{t = 1}^n {{\alpha _t}{\theta _t}} }}{{\sum\nolimits_{t = 1}^n {{\alpha _t}} }}
\quad , \quad
{\bar y _n} := \frac{{\sum\nolimits_{t = 1}^n {{\alpha _t}{y _t}} }}{{\sum\nolimits_{t = 1}^n {{\alpha _t}} }}
\label{eq:bartheta}
\end{equation}

\end{algorithmic}
\end{algorithm}
%


\subsection{Assumptions}

In this section, we make several assumptions on the MDP and basis functions that are used in our finite-sample analysis of the revised GTD algorithms. These assumptions are quite standard and are similar to those made in the prior work on  GTD algorithms~\cite{Sutton:GTD1:2008,tdc:2009,maei2011thesis} and those made in the analysis of SG algorithms~\cite{RobustSA:2009}.
\begin{assumption}
(\textbf{Feasibility Sets})
We define the bounded closed convex sets $\Theta\subset\mathbb{R}^d$ and $Y\subset\mathbb{R}^d$ as the feasible sets in Algorithm~\ref{alg:pgtd2}. We further assume that the saddle-point $(\theta^*,y^*)$ of the optimization problem~\eqref{eq:sp} belongs to $\Theta\times Y$. We also define $D_\theta:=\big[\max_{\theta\in\Theta}||\theta||_2^2-\min_{\theta\in\Theta}||\theta||_2^2\big]^{1/2}$, $D_y:=\big[\max_{y\in Y}||y||_2^2-\min_{y\in Y}||y||_2^2\big]^{1/2}$, and $R=\max\big\{\max_{\theta\in\Theta}||\theta||_2,\max_{y\in Y}||y||_2\big\}$. 
\label{ass:xyfeasible}
\end{assumption}
\begin{assumption}
(\textbf{Non-singularity})
We assume that the covariance matrix $C  = \mathbb{E}[{\phi_i}\phi_i^\top]$ and matrix $A=\mathbb{E}\big[\rho_i\phi_i(\Delta\phi_i)^\top\big]$ are non-singular.
\label{ass:C}
\end{assumption}
\begin{assumption} 
(\textbf{Boundedness})
\label{ass:bound}
{We assume} the features $(\phi_i,\phi^{'}_i)$ have uniformly bounded second moments. This together with the boundedness of features (by $L$) and importance weights (by $\rho_{\max}$) guarantees that the matrices $A$ and $C$, and vector $b$ are uniformly bounded. 
\end{assumption}
This assumption guarantees that for any $(\theta,y)\in\Theta\times Y$, the unbiased estimators of $b-A\theta-My$ and $A^\top y$, i.e.,
\begin{align}
\mathbb{E}[\hat{b}_t - \hat{A}_t\theta - \hat{M}_t y] &= b - A\theta  - My, \nonumber \\
\mathbb{E}[\hat{A}_t^\top y] &= {A^\top}y,
\end{align}
all have bounded variance, i.e.,
\begin{align}
\nonumber
\mathbb{E}\big[||\hat{b}_t - \hat{A}_t\theta - \hat{M}_t y - (b - A\theta  - My)|{|^2}\big] &\le \sigma _1^2, \nonumber \\
\mathbb{E}\big[||\hat{A}_t^\top y - {A^\top }y|{|^2}\big] &\le \sigma _2^2,
\label{eq:sigma123}
\end{align}
where $\sigma_1$ and $\sigma_2$ are non-negative constants. We further define 
\begin{equation}
\label{eq:sigma}
\sigma^2  = \sigma _1^2 + \sigma _2^2.
\end{equation}
Assumption~\ref{ass:bound} also gives us the following ``light-tail'' assumption. 
There exist constants ${M_{*,\theta }}$ and ${M_{*,y}}$ such that
\begin{align}
\label{eq:lt}
\nonumber
&\mathbb{E}[\exp \{ \frac{{||{{\hat b}_t} - {{\hat A}_t}\theta  - {{\hat M}_t}y|{|^2}}}{{M_{_{*,\theta }}^2}}\} ]  \le \exp \{ 1\}, \\
&\mathbb{E}[\exp \{ \frac{{||\hat A_t^ \top y|{|^2}}}{{M_{_{*,y}}^2}}\} ] \le \exp \{ 1\}. 
\end{align}
This ``light-tail'' assumption is equivalent to the assumption in Eq.~(3.16) in the work by~\citeauthor{RobustSA:2009}~\citeyear{RobustSA:2009} and is necessary for the high-probability bound of Proposition~\ref{pro:hp}. We will show how to compute ${M_{*,\theta }},{M_{*,y}}$ in the Appendix.


\subsection{Finite-Sample Performance Bounds}
\label{subsec:FSB}

The finite-sample performance bounds that we derive for the GTD algorithms in this section are for the case that the training set $\mathcal{D}$ has been generated as discussed in Section~\ref{sec:preliminaries}. We further discriminate between the on-policy ($\pi=\pi_b$) and off-policy ($\pi\neq\pi_b$) scenarios. The sampling scheme used to generate $\mathcal{D}$, in which the first state of each tuple, $s_i$, is an i.i.d.~sample from a distribution $\xi$, also considered in the original GTD and GTD2 papers {is} for the analysis of these algorithms, and not {used} in the experiments~\cite{Sutton:GTD1:2008,tdc:2009}. Another scenario that can motivate this sampling scheme is when we are given a set of high-dimensional data  generated either in an on-policy or off-policy manner, and $d$ is so large that the value function of the target policy cannot be computed using a least-squares method (that involves matrix inversion), and iterative techniques similar to GTD/GTD2 are required. 

We first derive a high-probability bound on the error function of the saddle-point problem~\eqref{eq:sp} at the GTD solution $(\bar{\theta}_n,\bar{y}_n)$. Before stating this result in Proposition~\ref{pro:hp}, we report the following lemma that is used in its proof.

\begin{lemma}
\label{lem:abbound}
The induced $\ell_2$-norm of matrix $A$ and the $\ell_2$-norm of vector $b$ are bounded by 
\begin{align}
||A||_2
 \le  (1 + \gamma ){\rho _{\max }}L^2 d,\quad\;
||b||_2
 \le  {\rho _{\max }}L{{\rm{R}}_{\max }}. 
 \label{eq:abbound}
\end{align}
\end{lemma}
\begin{proof}
See the Appendix. 
\end{proof}

\begin{proposition}
\label{pro:hp}
Let $(\bar{\theta}_n,\bar{y}_n)$ be the output of the GTD algorithm after $n$ iterations (see Eq.~\eqref{eq:bartheta}). Then, with probability at least $1-\delta$, we have 
\begin{align}
\label{eq:hp}
{\rm{Err}}({\bar \theta _n}&,{\bar y_n}) \le \sqrt {\frac{5}{n}} (8 + 2\log \frac{2}{\delta }){R^2} \\
&\times \left(\rho_{\max}L\Big(2(1 + \gamma )Ld + \frac{R_{\max}}{R}\Big) + \tau + \frac{\sigma }{R}\right), \nonumber
\end{align}
where ${\rm{Err}}({\bar \theta _n},{\bar y_n})$ is the error function of the saddle-point problem~\eqref{eq:sp} defined by Eq.~\eqref{eq:errdef}, $R$ {is} defined in Assumption~\ref{ass:xyfeasible}, $\sigma$ is from Eq.~\eqref{eq:sigma}, and $\tau=\sigma_{\max}(M)$ is the largest singular value of $M$, which means $\tau= 1$ for GTD and $\tau=\sigma_{\max}(C)$ for GTD2.
\end{proposition}
\begin{proof}
We give a proof sketch here.
The proof of Proposition~\ref{pro:hp} heavily relies on Proposition~3.2 in the work by~\citeauthor{RobustSA:2009}~\citeyear{RobustSA:2009}. We just need to map our convex-concave {\em stochastic} saddle-point problem~\eqref{eq:sp}, i.e., 
\begin{equation}
\mathop{\min}\limits_{\theta\in\Theta}\mathop{\max}\limits_{y\in Y}\left({L(\theta,y)=\left\langle {b-A\theta,y}\right\rangle -\frac{1}{2}||y||_{M}^{2}}\right)
\end{equation}
to the form in Proposition~3.2. The details of verifying the conditions are in the Appendix. Note that in the work of~\citeauthor{RobustSA:2009}~\citeyear{RobustSA:2009}, the robust stochastic approximation technique is used, mainly by combining aggressive step-sizes (i.e., large constant step-sizes) and iterative averaging~\cite{polyak1992acceleration} (also termed as \textit{Polyak's averaging}). The choice of the constant step-size is described in the Appendix and the iterative averaging is shown in Eq.~\eqref{eq:bartheta}.
\end{proof}

\begin{theorem}
\label{thm:1}
Let $\bar{\theta}_n$ be the output of the GTD algorithm after $n$ iterations (see Eq.~\eqref{eq:bartheta}). Then, with probability at least $1-\delta$, we have 
\begin{equation}
\frac{1}{2}||A{\bar \theta _n} - b||_{\xi} ^2 \le \tau {\xi _{\max }}\;{\rm{Err}}({\bar \theta _n},{\bar y_n}).
\label{eq:thm1}
\end{equation}
\end{theorem}
\begin{proof}
From Proposition~\ref{pro:1}, for any $\theta$, we have 
\begin{equation}
\max_{y}\;L(\theta,y) = \frac{1}{2}||A\theta  - b||_{M^{-1}}^2.
\end{equation}
Given Assumption~\ref{ass:C}, the system of linear equations $A\theta=b$ has a solution $\theta^*$, i.e., the (off-policy) fixed-point $\theta^*$ exists, and thus, we may write 
\begin{align}
\min_\theta\;\max_y\;L(\theta,y) &= \min_\theta\;\frac{1}{2}||A\theta - b||_{M^{-1}}^2 \\ 
&= \frac{1}{2}||A{\theta^*} - b||_{{M^{ - 1}}}^2 = 0.
\end{align}
In this case, we also have\footnote{We may write the second inequality as an equality for our saddle-point problem defined by Eq.~\eqref{eq:sp}.}
\begin{align}
\min_\theta\;L(\theta,y) &\le \max_y\;\min_\theta\;L(\theta,y) \le \min_\theta\;\max_y\;L(\theta,y) \nonumber \\
&= \frac{1}{2}||A\theta^* - b||_{M^{-1}}^2 = 0.
\label{eq:thm1-1}
\end{align}
From Eq.~\eqref{eq:thm1-1}, for any $(\theta,y)\in\Theta\times Y$ including $(\bar{\theta}_n,\bar{y}_n)$, we may write
\begin{align}
\label{eq:err1}
{\rm{Err}}(\bar{\theta}_n,\bar{y}_n) &= \max_y\;L(\bar{\theta}_n,y) - \min_\theta\;L(\theta,\bar{y}_n) \\ 
&\ge \max_y\;L(\bar{\theta}_n,y) = \frac{1}{2}||A\bar{\theta}_n - b||_{M^{ - 1}}^2. \nonumber
\end{align}
Since $||A\bar{\theta}_n - b||_{\xi}^2 \le \tau \xi_{\max}\;||A\bar{\theta}_n - b||_{M^{ - 1}}^2$, where $\tau$ is the largest singular value of $M$, we have
\begin{equation}
\frac{1}{2}||A{\bar \theta _n} - b||_{\xi} ^2 \le \frac{{\tau {\xi _{\max }}}}{2}||A{\bar \theta _n} - b||_{{M^{ - 1}}}^2 \le \tau {\xi _{\max }}\;{\rm{Err}}({\bar \theta _n},{\bar y_n}).
\label{eq:err3} 
\end{equation}
The proof follows by combining Eq.~\eqref{eq:err3} and Proposition~\ref{pro:hp}.
\end{proof}

With the results of Proposition~\ref{pro:hp} and Theorem~\ref{thm:1}, we are now ready to derive finite-sample bounds on the performance of GTD/GTD2 in both on-policy and off-policy settings.


\subsubsection{On-Policy Performance Bound}

In this section, we consider the on-policy setting in which the behavior and target policies are equal, i.e.,~$\pi_b=\pi$, and the sampling distribution $\xi$ is the stationary distribution of the target policy $\pi$ (and the behavior policy $\pi_b$). We use Lemma~\ref{lem:v} to derive our on-policy bound. The proof of this lemma can be found in the work of~\citeauthor{DantzigRL:2012}~\citeyear{DantzigRL:2012}.

\begin{lemma}
\label{lem:v}
For any parameter vector $\theta$ and corresponding $\hat v = \Phi \theta $, the following equality holds 
\begin{equation}
V - \hat v = {(I - \gamma \Pi P)^{ - 1}}\left[ {\left( {V - \Pi V} \right) + \Phi {C^{ - 1}}(b - A\theta )} \right].
\label{eq:vbound}
\end{equation}
\end{lemma}

Using Lemma~\ref{lem:v}, we derive the following performance bound for GTD/GTD2 in the on-policy setting.

\begin{proposition}
\label{pro:4}
Let $V$ be the value of the target policy and ${{\bar v}_n} = \Phi {{\bar \theta }_n}$, where $\bar{\theta}_n$ defined by Eq.~\eqref{eq:bartheta}, be the value function returned by on-policy GTD/GTD2. Then, with probability at least $1-\delta$, we have
\begin{equation}
||V - {\bar v_n}|{|_\xi } \le \frac{1}{{1 - \gamma }}\left( {||V - \Pi V|{|_\xi } + \frac{L}{\nu }\sqrt {2d\tau {\xi _{\max }}{\rm{Err}}({{\bar \theta }_n},{{\bar y}_n})} } \right)
\label{eq:pro4}
\end{equation}
%
%
where $\rm{Err}(\bar \theta_n,\bar y_n)$ is upper-bounded by Eq.~\eqref{eq:hp} in Proposition~\ref{pro:hp}, with $\rho_{\max}=1$ (on-policy setting).
\end{proposition}
\begin{proof}
See the Appendix. 
\end{proof}

\noindent \textbf{Remark:}
It is important to note that Proposition~\ref{pro:4} shows that the error in the performance of the GTD/GTD2 algorithm in the on-policy setting is of $O\left( {\frac{{{L^2}d\sqrt {\tau {\xi _{\max }}\log \frac{1}{\delta }} }}{{{n^{1/4}\nu}}}} \right)$. Also note that the term $\frac{\tau}{\nu}$ in the GTD2 bound is the conditioning number of the covariance matrix $C$.


\subsubsection{Off-Policy Performance Bound}

In this section, we consider the off-policy setting in which the behavior and target policies are different, i.e.,~$\pi_b\neq\pi$, and the sampling distribution $\xi$ is the stationary distribution of the behavior policy $\pi_b$. We assume that off-policy fixed-point solution exists, i.e.,~there exists a $\theta^*$ satisfying $A\theta^*=b$. Note that this is a direct consequence of Assumption~\ref{ass:C} in which we assumed that the matrix $A$ in the off-policy setting is non-singular. We use Lemma~\ref{lem:kolter} to derive our off-policy bound. The proof of this lemma can be found in the work of~\citeauthor{Kolter:offpolicyTD}~\citeyear{Kolter:offpolicyTD}. Note that $\kappa (\bar D)$ in his proof is equal to $\sqrt{\rho_{\max}}$ in our paper.
\begin{lemma}
\label{lem:kolter}
If $\Xi$ satisfies the following linear matrix inequality 
\begin{align}
\left[ {\begin{array}{*{20}{c}}
{{\Phi ^\top}\Xi \Phi }&{{\Phi ^\top}\Xi P\Phi }\\
{{\Phi ^\top}{P^\top}\Xi \Phi }&{{\Phi ^\top}\Xi \Phi }
\end{array}} \right] \succeq 0
\label{eq:kolterlmi}
\end{align}
and let $\theta^*$ be the solution to $A\theta^* = b$, 
 then we have
\begin{align}
||V - \Phi \theta^* |{|_{\xi} } \le \frac{{1 + \gamma \sqrt {{\rho _{\max }}} }}{{1 - \gamma }}||V - \Pi V|{|_{\xi} }.
\label{eq:kolter2}
\end{align}
\end{lemma}


Note that the condition on $\Xi$ in Eq.~\eqref{eq:kolterlmi} guarantees that the behavior and target policies are not too far away from each other. Using Lemma~\ref{lem:kolter}, we derive the following performance bound for GTD/GTD2 in the off-policy setting.

\begin{proposition}
\label{pro:5}
Let $V$ be the value of the target policy and ${{\bar v}_n} = \Phi {{\bar \theta }_n}$, where $\bar{\theta}_n$ is defined by~\eqref{eq:bartheta}, be the value function returned by off-policy GTD/GTD2. Also let the sampling distribution $\Xi$ {satisfy} the condition in Eq.~\eqref{eq:kolterlmi}. Then, with probability at least $1-\delta$, we have
\begin{align}
\label{eq:pro5}
||V - {\bar v_n}|{|_\xi } & \le \frac{{1 + \gamma \sqrt {{\rho _{\max }}} }}{{1 - \gamma }}||V - \Pi V|{|_\xi }\\
\nonumber
&+ \sqrt {\frac{{2{\tau _C}\tau {\xi _{\max }}}}{{{\sigma _{\min }}({A^ \top }{M^{ - 1}}A)}}{\rm{Err}}({{\bar \theta }_n},{{\bar y}_n})},
\end{align}
where ${\tau _C} = \sigma_{\max} (C)$. 
\end{proposition}
\begin{proof}
See the Appendix. 
\end{proof}

\subsection{Accelerated Algorithm}

As discussed at the beginning of Section~\ref{sec:saddle-point}, this saddle-point formulation not only gives us the opportunity to use the techniques for the analysis of SG methods to derive finite-sample performance bounds for the GTD algorithms, as we showed in Section~\ref{sec:analysis}, but {it also} allows us to use the powerful algorithms that have been recently developed to solve the SG problems and derive more efficient versions of GTD and GTD2. Stochastic Mirror-Prox (SMP)~\cite{juditsky2008solving} is an ``almost dimension-free'' non-Euclidean extra-gradient method that deals with both smooth and non-smooth stochastic optimization problems (see~\citeR{sra2011optimization} and~\citeR{bubeck2014optml} for more details). Using SMP, we propose a new version of GTD/GTD2, called GTD-MP/GTD2-MP, with the following update formula:\footnote{For simplicity, we only describe mirror-prox GTD methods where the mirror map is identity, which can also be viewed as extragradient (EG) GTD methods. \citeauthor{proximalrl}~\citeyear{proximalrl} give a more detailed discussion of a  broad range of mirror maps in RL.} 
\begin{align}
y_t^m &= y_t + \alpha _t(\hat b_t - \hat  A_t\theta _t - \hat  M_ty_t), \quad\;\;\; \theta_t^m = \theta_t + \alpha_t \hat A_t^\top y_t, \\
y_{t + 1} &= y_t + \alpha_t(\hat b_t - \hat A_t\theta_t^m - \hat M_ty_t^m), \; \theta_{t + 1} = \theta_t + \alpha_t \hat A_t^\top y_t^m.
\end{align}
After $T$ iterations, these algorithms return ${\bar \theta _T}: = \frac{{\sum\nolimits_{t = 1}^T {{\alpha _t}{\theta _t}} }}{{\sum\nolimits_{t = 1}^T {{\alpha _t}} }}$ and ${\bar y_T}: = \frac{{\sum\nolimits_{t = 1}^T {{\alpha _t}{y_t}} }}{{\sum\nolimits_{t = 1}^T {{\alpha _t}} }}$. 
The details of the algorithm are shown in Algorithm~\ref{alg:GTD2-MP}, and the experimental comparison study between GTD2 and GTD2-MP is reported in Section~\ref{sec:exp}.

\begin{algorithm}
\caption{GTD2-MP}
\label{alg:GTD2-MP} 
\begin{algorithmic}[1]
\FOR {$t=1,\ldots,n$}
\STATE Update parameters

\begin{align}
\nonumber
{\delta_{t}(\theta_t)} &= {r_t} - \theta _t^\top \Delta {\phi _t}\\
\nonumber
y_t^m &= {y_t} + {\alpha _t}(\rho_t{\delta _t} - \phi _t^\top {y_t}){\phi _t}\\
\nonumber
\theta _t^m &= {\theta _t} + {\alpha _t}\rho_t\Delta {\phi _t}(\phi _t^\top {y_t})\\
\nonumber
\delta _t^m (\theta _t^m) &= {r_t} - (\theta _t^m)^\top \Delta {\phi _t}\\
\nonumber
{y_{t + 1}} &= {y_t} + {\alpha _t}(\rho_t\delta _t^m - \phi _t^\top y_t^m){\phi _t}\\
\nonumber
{\theta _{t + 1}} &= {\theta _t} + {\alpha _t}\rho_t\Delta {\phi _t}(\phi _t^\top y_t^m)
\end{align}

\ENDFOR
\STATE OUTPUT

\begin{equation}
{\bar \theta _n} := \frac{{\sum\nolimits_{t = 1}^n {{\alpha _t}{\theta _t}} }}{{\sum\nolimits_{t = 1}^n {{\alpha _t}} }}
\quad , \quad
{\bar y _n} := \frac{{\sum\nolimits_{t = 1}^n {{\alpha _t}{y _t}} }}{{\sum\nolimits_{t = 1}^n {{\alpha _t}} }}
\end{equation}

\end{algorithmic}
\end{algorithm}

\section{Further Analysis}
In this section, we discuss different aspects of the proximal gradient TD framework from several perspectives, such as acceleration, learning with inexact importance weight factor $\rho_t$, finite-sample analysis with Markov sampling condition, and discussion of TDC algorithm.

\subsection{Acceleration Analysis}
\label{sec:acceleration}
In this section, we are going to discuss the convergence rate of the accelerated algorithms using off-the-shelf accelerated solvers for saddle-point problems. For simplicity, we will discuss the error bound of $\frac{1}{2}||A\theta  - b||_{{M^{ - 1}}}^2$, and the corresponding error bound of $\frac{1}{2}||A\theta  - b||_{\xi}^2$ and $\|V - {\bar v}_n|{|_{\xi} }$ can be likewise derived.
As can be seen from the above analysis, the convergence rate of the GTD algorithms family is
\begin{equation}
({\bf{GTD}}/{\bf{GTD2}}):\quad O\left( {\frac{{\tau  + ||A|{|_2}  + \sigma }}{{\sqrt n }}} \right).
\end{equation}
In this section, we raise an interesting question: what is the ``optimal" GTD algorithm? To answer this question, we review the convex-concave formulation of GTD2. 
According to convex programming complexity theory~\cite{juditsky2008solving},  the un-improvable convergence rate of {the} stochastic saddle-point problem~\eqref{eq:sp} is 
%
%

\begin{equation}
({\bf{Optimal}}):\quad O\left( {\frac{{\tau}}{{{n^2}}} + \frac{{||A|{|_2}}}{n} + \frac{\sigma }{{\sqrt n }}} \right).
\label{eq:rateoptimal}
\end{equation}
There are many readily available stochastic saddle-point solvers, such as {the} stochastic Mirror-Prox  (SMP)~\cite{juditsky2008solving} algorithm, which leads to our proposed GTD2-MP algorithm.  GTD2-MP is able to accelerate the convergence rate of our gradient TD method to: 
%
\begin{equation}
({\bf{GTD2-MP}}):\quad O\left( {\frac{{\tau  + ||A|{|_2} }}{n} + \frac{\sigma }{{\sqrt n }}} \right).
\end{equation}


\subsection{Learning with Biased $\rho_t$}

The importance weight factor $\rho_t$ is lower bounded by $0$, but yet may have an arbitrarily large upper bound.
In real applications, the importance weight factor $\rho_t$ may not be estimated exactly, i.e., the estimation $\hat{\rho}_t$ is a biased estimation of the true $\rho_t$.  To this end, the stochastic gradient we obtained is not the unbiased gradient of $L(\theta,y)$ anymore.
  This falls into a broad category of learning with inexact stochastic gradient, or termed as stochastic gradient methods with an inexact oracle~\cite{inexact:stochastic:devolder2011}. Given the inexact stochastic gradient, the convergence rate and performance bound become much worse than the results with exact stochastic gradient. Based on the analysis by~\citeauthor{juditsky2008solving}~\citeyear{juditsky2008solving}, we have the error bound for inexact estimation of $\rho_t$.
\begin{proposition}
Let ${{\bar{\theta}  }_n}$ be defined as above. Assume at the $t$-th iteration, $\hat{\rho}_t$ is the estimation of the importance weight factor $\rho_t$ with bounded bias such that 
$
\mathbb{E}[ \hat{\rho}_t  - {\rho _t} ] \le \epsilon. 
$ 
The convergence rates of GTD/GTD2 algorithms with iterative averaging {are} as follows,
\begin{equation}
 ||A{\bar \theta _n} - b||_{{M^{ - 1}}}^2 \le O\left( {\frac{{\tau  + ||A||_2  + \sigma }}{{\sqrt n }}} \right) +O(\epsilon).
\label{eq:biasrho}
\end{equation}
\end{proposition}
This implies that the inexact estimation of $\rho_t$ may  cause disastrous estimation error, which implies that an exact estimation of $\rho_t$ is very important.


\subsection{Finite-Sample Analysis of Online Learning}
Another more challenging scenario is the online learning scenario, where the samples are interactively generated by the environment, or by an interactive agent. The difficulty lies in that the sample distribution does not follow the i.i.d sampling condition anymore, but follows an underlying Markov chain $\mathcal{M}$. If the Markov chain $\mathcal{M}$'s mixing time is small enough, i.e., the sample distribution reduces to the stationary distribution of $\pi_b$ very fast, our analysis still applies. However, it is usually the case that the underlying Markov chain's mixing time $\tau_{\rm{mix}}$ is not small enough. The analysis can be conducted by extending the result of recent work~\cite{duchi2012ergodic} from strongly convex loss functions to saddle-point problems. 
Following this line of research, \citeauthor{td-finite:wang:2017}~\citeyear{td-finite:wang:2017} conducted the finite-sample analysis of GTD2(0) algorithms in the Markov noise setting, which is the same in convergence rate order but different in the constant factors. 

\subsection{Discussion of TDC Algorithm}
\label{sec:tdc}

Now we discuss the limitation of our analysis with regard to the temporal difference with correction (TDC) algorithm~\cite{tdc:2009}. 
Interestingly, the TDC algorithm seems not to have an explicit saddle-point representation, since it incorporates the information of the optimal $y_{t}^{*}(\theta_{t})$ into the update of $\theta_t$, a quasi-stationary condition which is commonly used in two-time-scale stochastic approximation approaches.
 An intuitive answer to the advantage of TDC over GTD2 is that
the TDC update of $\theta_{t}$ can be considered as incorporating
the prior knowledge into the update rule: for a stationary $\theta_{t}$,
if the optimal  $y_{t}^{*}(\theta_{t})$
has a closed-form solution or is easy to compute, then incorporating
this $y_{t}^{*}(\theta_{t})$ into the update law tends to accelerate
the algorithm's convergence performance. For the GTD2 update, note that there is a sum of two terms where $y_{t}$
appears, which are $\rho_t({\phi_{t}}-\gamma\phi_{t}^{\prime})(y_{t}^\top{\phi_{t}})={\rho_t\phi_{t}}(y_{t}^\top{\phi_{t}})-\gamma\rho_t\phi_{t}^{\prime}(y_{t}^\top{\phi_{t}})$. Replacing $y_{t}$ in the first term with $y^*_{t}(\theta_t)=\mathbb{E}{[{\phi_{t}}\phi_{_{t}}^\top]^{-1}}\mathbb{E}[\rho_t{\delta_{t}}(\theta_t){\phi_{t}}]$,
we have the TDC update rule. 
There are two key factors that impedes the finite-sample analysis of TDC algorithm with the saddle-point approach.
Firstly, in contrast to GTD/GTD2, TDC is a two-time scale algorithm where $\mathop {\lim }\limits_{t \to \infty } \frac{{{\alpha _t}}}{{{\beta _t}}} = 0$. Secondly, note that TDC does not minimize \emph{any} objective functions, thus does not have a stochastic primal-dual formulation as GTD and GTD2, and the asymptotic convergence of TDC requires more restrictions than GTD2 as shown by~\citeauthor{tdc:2009}~\citeyear{tdc:2009}.

\subsection{Recent Results of Related Work}
The proximal gradient temporal difference learning framework introduces many stochastic optimization techniques that facilitate theoretical analysis of reinforcement learning algorithms, such as Polyak's iterative averaging, projections, and constant stepsizes, which were first introduced by~\citeauthor{ROTD:NIPS2012}~\citeyear{ROTD:NIPS2012}.
We briefly review several significant research advances on the finite-sample analysis of linear temporal difference learning algorithms since the finite-sample analysis of GTD algorithms which was first published~\cite{liu2015uai}. 
The first line of work is the analysis of the GTD algorithm family~\cite{td-finite:dalal:2018,td-finite:colt:dalal:2018} with different learning settings. ~\citeauthor{td-finite:colt:dalal:2018}~\citeyear{td-finite:colt:dalal:2018} conducted a finite-sample analysis of the two-time-scale GTD, GTD2, and TDC algorithms using a concentration bound for stochastic approximation methods via Alekseev's Formula~\cite{td-finite:kamal:2010,td-finite:thoppe:2015}.
Using this approach, the convergence rate of GTD2 w.r.t \textit{mean-square error} (MSE) $||V-\hat{v}_n||^2_{\xi}$ proposed in the work of~\citeauthor{td-finite:colt:dalal:2018}~\citeyear{td-finite:colt:dalal:2018} is $\tilde O({n^{ - (1 - \chi )\frac{2}{3}}})$, where $\chi$ is a tuning parameter that influences the stepsizes used by the algorithms, and a special ``sparse projection'' is used. This framework enables the finite-sample analysis of the TDC algorithm, which cannot be analyzed from the saddle-point perspective, as explained in Section~\ref{sec:tdc}.
The other work by~\citeauthor{td-finite:wang:2017}~\citeyear{td-finite:wang:2017} investigates the convergence rate assuming a Markov sampling condition and Robbins-Monro stepsizes.
Table~\ref{tab:comparison} presents a comparison of existing approaches. It should be noted that all of the analyses are for algorithms employing projections and iterative averaging.
\citeauthor{td-finite:Csaba:2018}~\citeyear{td-finite:Csaba:2018} also studied the impact of constant stepsizes and Polyak's iterative averaging on the TD algorithm with projections in the i.i.d setting, yet the result has not been extended to GTD algorithm family yet. 

The second line of work aims to adopt new stochastic saddle-point solvers into the proximal gradient TD framework and proposes new algorithms for acceleration, regularization and variance reduction. \citeauthor{proximalrl}~\citeyear{proximalrl} investigated proximal gradient TD with an $\ell_1$-regularizer to enhance sparsity. 
\citeauthor{gtd:du:2017stochastic}~\citeyear{gtd:du:2017stochastic} introduced the stochastic variance reduced saddle-point solver~\cite{saddlepoint:palaniappan:2016stochastic} to reach a linear convergence rate for a fixed set of samples.
The third line of work focuses on the analysis of gradient temporal difference learning algorithms with eligibility traces~\cite{gtd:lambda:yu:2017}. To the best of our knowledge, the only asymptotic convergence analysis for GTD($\lambda$) with $\lambda>0$ is proposed by~\citeauthor{gtd:lambda:yu:2017}~\citeyear{gtd:lambda:yu:2017}, and the finite-sample analysis still remains an open problem.
Comprehensive studies of different temporal difference learning algorithms with eligibility traces have been conducted in the works of~\citeauthor{dann2014tdsurvey}~\citeyear{dann2014tdsurvey} and~\citeauthor{td:comparison:adam:2016}~\citeyear{td:comparison:adam:2016}.
Last but not least, Legendre-Fenchel duality has also been used in other reinforcement learning problems, such as mean-variance policy search~\cite{liubo:nips:2018}, where the original mean-variance objective function is formulated as a coordinate ascent problem by introducing Legendre-Fenchel duality.

\begin{tiny}
\begin{table}
\centering
\begin{tabular}{|l|l|l|l|}
\hline 
\textbf{Method} 
&\textbf{GTD2~{\tiny \cite{liu2015uai}}} &\textbf{GTD2~{\tiny\cite{td-finite:wang:2017}}} &\textbf{GTD2~{\tiny\cite{td-finite:colt:dalal:2018}}} 
\\\hline\hline
\textbf{Sampling} & i.i.d & Markov & i.i.d
\\\hline 
\textbf{Stepsize's time-scale} & single-time-scale  & single-time-scale  & two-time-scale 
\\\hline 
\textbf{Stepsize choice} & constant  & Robbins-Monro  & 
$\begin{array}{*{20}{c}}
{{\alpha _n} = 1/{n^{1 - \chi }},}\\
{{\beta _n} = 1/{n^{(1 - \chi )\frac{2}{3}}}}
\end{array}$ 
\\\hline 
\textbf{Projection} & Y & Y & Y
\\\hline 
\textbf{Iterative averaging} & Y  & Y & Y
\\\hline 
\textbf{Conv. rate of MSE} &
$\tilde{O}(1/{\sqrt n})$ &
$\tilde O( \frac{\Sigma_{t=1}^n {\alpha_t^2}}{\Sigma_{t=1}^n {\alpha_t}} )$ &
$\tilde O({n^{ - 2(1 - \chi )/3}})$
\\\hline 
\end{tabular}
\caption{Comparison of Existing Methods}
\label{tab:comparison}
\end{table}
\end{tiny}

\section{Control Learning Extension}

In this section, we are going to discuss the control learning extension of the family of proximal gradient algorithms. To this end, we will first present a lemma bridging the connection between {the} forward-view and backward-view perspectives. Then based on {the} GQ~\cite{gq:maei2010} algorithm, we propose the control learning extension of the GTD2-MP algorithm, which is termed as {the} GQ-MP algorithm.


\subsection{Extension to Eligibility Trace}

The $T$ operator looks one-step ahead, {but} it would be beneficial to look multiple steps ahead~\cite{sutton-barto:book}, which gives rise to the \textit{multiple-step Bellman operator} $T^{\lambda}$, {otherwise known} as {the} $\lambda$-weighted Bellman operator~\cite{sutton-barto:book}, which is an 
arithmetic mean of the power series of $T$, i.e., 
\begin{equation}
{T^\lambda } = (1 - \lambda )\sum\nolimits_{i = 0}^\infty  {{\lambda ^i}{T^{i + 1}}} ,\lambda  \in (0,1).
\end{equation}
Correspondingly, the \textit{multiple-step TD error}, also termed as the  $\lambda$-TD error ${\delta ^\lambda }$ w.r.t ${\theta}$ is defined as
\begin{equation}
\mathbb{E}[{\delta ^\lambda }(\theta )]   = {T^\lambda }\hat v - \hat v = {T^\lambda }\Phi \theta  - \Phi \theta .
\label{eq:deltalambda}
\end{equation}
The objective function in Eq.~\eqref{eq:j} is changed accordingly as follows by replacing $T$ with $T^{\lambda}$ , 
\begin{equation}
\label{eq:obj:fv}
{\rm J}(\theta ) = ||{\Phi ^{\top}}\Xi ({T^\lambda }{\hat v } - {\hat v })||_{{M^{ - 1}}}^2 = ||\mathbb{E}[{\rho_i}{\phi _i}\delta _i^\lambda (\theta )]||_{{M^{ - 1}}}^2.
\end{equation}
This is called {the} \emph{forward view} since it calls for looking multiple steps ahead, which is difficult to implement in practice. To this end, the \emph{backward view} using eligibility traces is easy to implement. The eligibility trace is defined in a recursive way as
\begin{align}
e_0 =& 0\\
\nonumber
e_t =& \rho_t \gamma \lambda e_{t-1} + \phi_t.
\end{align}
 We will introduce Theorem~\ref{thm:equiv} to bridge the gap between the backward and forward view. 
\begin{theorem}
\cite{maei2011thesis,geist2014off:trace:jmlr}
There is {an} equivalence between the forward view and backward view such that
\begin{equation}
\mathbb{E}[{\phi _i}\delta _i^\lambda (\theta )] = \mathbb{E}[{e_i}{\delta _i}(\theta )].
\label{eq:equiv}
\end{equation}
\label{thm:equiv}
\end{theorem}
The details of the forward view and the backward view can be seen in the book by~\citeauthor{sutton-barto:book}~\citeyear{sutton-barto:book}, Theorem 11 in the work of~\citeauthor{maei2011thesis}~\citeyear{maei2011thesis}, and Proposition 6 in the work of~\citeauthor{geist2014off:trace:jmlr}~\citeyear{geist2014off:trace:jmlr}. A natural extension to Eq.~\eqref{eq:equiv} multiplies the importance ratio factor on both sides of the equality as follows,
\begin{equation}
\mathbb{E}[{\rho_i}{\phi _i}\delta _i^\lambda (\theta )] = \mathbb{E}[{\rho_i}{e_i}{\delta _i}(\theta )].
\label{eq:equiv_rho}
\end{equation}

\subsection{Greedy-GQ($\lambda$) Algorithm}
With the help of Theorem~\ref{thm:equiv}, we can convert the objective formulation in Eq.~\eqref{eq:obj:fv} {to}
\begin{equation}
\label{eq:obj:bv}
J(\theta ) = ||\mathbb{E}[{\rho_i}{e_i}{\delta _i}(\theta )]||_{{M^{ - 1}}}^2.
\end{equation}
{The} corresponding primal-dual formulation is
\begin{equation}
J(\theta ) = \mathop {\max }\limits_y \left( {\left\langle {\mathbb{E}[{\rho_i}{e_i}{\delta _i}(\theta )],y} \right\rangle  - \frac{1}{2}||y||^2_M } \right)
\label{eq:obj:bv:pd}
\end{equation}
and thus the new algorithm can be derived as
\begin{align}
\nonumber
{\theta _{t + 1}} =& {\theta _t} + {\alpha _t}{\rho_t}\Delta {\phi _t}(e_t^{\top}{y_t})\\
{y_{t + 1}}       =& {y_t} + {\alpha _t}\left( {{\rho_t}{\delta _t}{e_t} - {M_t}{y_t}} \right).
\label{eq:gtd:trace}
\end{align}
Correspondingly, Greedy-GQ($\lambda$) with importance sampling can be derived. Here we present the Greedy-GQ($\lambda$) algorithm.
\begin{algorithm}
\caption{Greedy-GQ($\lambda$)}
\label{alg:GQmp} 
\noindent Initialize $e_t = 0$, starting from $s_{0}$.
\begin{algorithmic}[1]
\REPEAT
\STATE Take $a_t$ according to $\pi_b$, and arrive at $s_{t+1}$
\STATE Compute $a_t^* = \arg {\max _a}{\theta ^\top}\phi ({s_t},a)$. If $a_t=a_t^*$, then 
${\rho _t} = \frac{1}{{{\pi _b}({a_t}|{s_t})}}$; otherwise $\rho_t=0$.    
\STATE Compute $\theta_{t+1},y_{t+1}$ according to \textbf{GQ-MP-LEARN} Algorithm.
\STATE Choose action $a_{t}$, and get $s_{t+1}, r_{t+1}$
\STATE Set $t\leftarrow t+1$;
\UNTIL $s_{t}$ is an absorbing state;
\STATE Compute $\bar{\theta}_{t},\bar{y}_{t}$ 
\end{algorithmic}
\end{algorithm}
and \textbf{GQ-MP-LEARN} algorithm, which is the core step of Greedy-GQ algorithm.
%
\begin{algorithm}
\caption{GQ-MP-LEARN}
\label{alg:GQmpLEARN} 
\begin{align*}
\nonumber
{e_t} &= \gamma \lambda \rho_t {e_{t - 1}} + \phi_{t}
\\
\nonumber
{\delta_{t}} &= {r_t} + \theta _t^ \top \Delta {\phi _t}\\
\nonumber
y_t^m &= {y_t} + {\alpha _t}\left( {{\rho_t}{e_t}{\delta _t} - (\phi _t^ \top {y_t}){\phi _t}} \right)\\
\nonumber
\theta _t^m &= {\theta _t} + {\alpha _t}{\rho_t}\Delta {\phi _t}(e_t^{\top}{y_t})\\
\nonumber
\delta _t^m &= {r_t} + \theta _t^{m \top }\Delta {\phi _t}\\
\nonumber
{y_{t + 1}} &= {y_t} + {\alpha _t}\left( {{\rho_t}{e_t}{\delta _t^m} - (\phi _t^ \top {y^m_t}){\phi _t}} \right)\\
\nonumber
\theta _t^m &= {\theta _t} + {\alpha _t}{\rho_t}\Delta {\phi _t}(e_t^{\top}{y^m_t})\\
\end{align*}
\end{algorithm}


\section{Empirical Evaluation}
\label{sec:exp}
In this section, we compare the previous GTD2 method with our proposed GTD2-MP method using various domains with regard to their value function approximation performance. 

\subsection{Baird Domain}
The Baird example \cite{Baird:ResidualAlgorithms1995} is a well-known example to test the performance of off-policy convergent algorithms. 
Constant stepsizes $\alpha = 0.005$ for GTD2 and $\alpha = 0.004$ for GTD2-MP are   chosen via comparison studies as in the work of~\citeauthor{dann2014tdsurvey}~\citeyear{dann2014tdsurvey}.
Figure~\ref{fig:star} shows the MSPBE curve of GTD2, GTD2-MP of $8000$ steps averaged over $200$ runs. We can see that GTD2-MP {gives} a significant improvement over the GTD2 algorithm
wherein both the MSPBE and variance are substantially reduced.
\begin{figure}[h]
\centering
\includegraphics[width= 1\textwidth, height=2.5in]{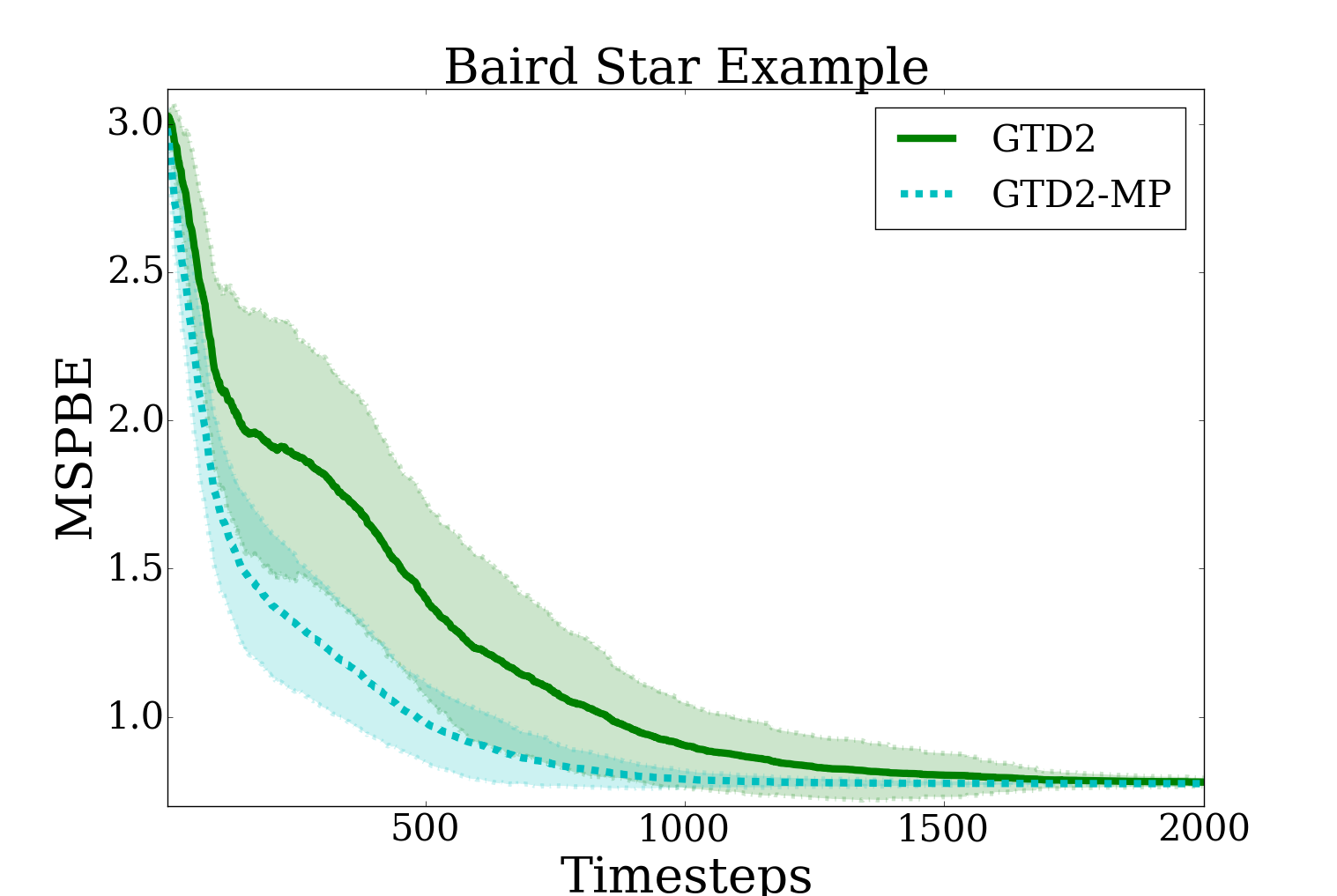}
\caption{Off-Policy Convergence Comparison}
\label{fig:star}
\end{figure}


\subsection{$50$-State Chain Domain}

The 50 state chain~\cite{lagoudakis:jmlr} is a standard MDP domain. There are 50 discrete
states $\{ {s_i}\} _{i = 1}^{50}$ and two actions moving the agent left ${s_i} \to {s_{\max ({{i - 1}},1)}}$ and right ${s_i} \to {s_{\min ({{i + 1}},50)}}$. The actions succeed with probability $0.9$; failed actions move the agent in the opposite direction. The discount factor is $\gamma = 0.9$. The agent receives a reward of $+1$ when in states $s_{10}$ and $s_{41}$. All other states have a reward of $0$.
In this experiment, we compare the performance of the value approximation w.r.t different stepsizes $\alpha  = 0.0001,0.001,0.01,0.1,0.2, \cdots ,0.9$ using the BEBF basis~\cite{bebf:parr:icml07}.  Figure~\ref{fig:chain} shows the value function approximation result where the cyan curve is the true value function, the red dashed curve is the GTD result, and the black curve is the GTD2-MP result. From the figure, one can see that GTD2-MP is much more robust w.r.t. stepsize choice than the GTD2 algorithm.

\begin{figure}
\centering
\includegraphics[width= 0.8\textwidth, height=2.5in]{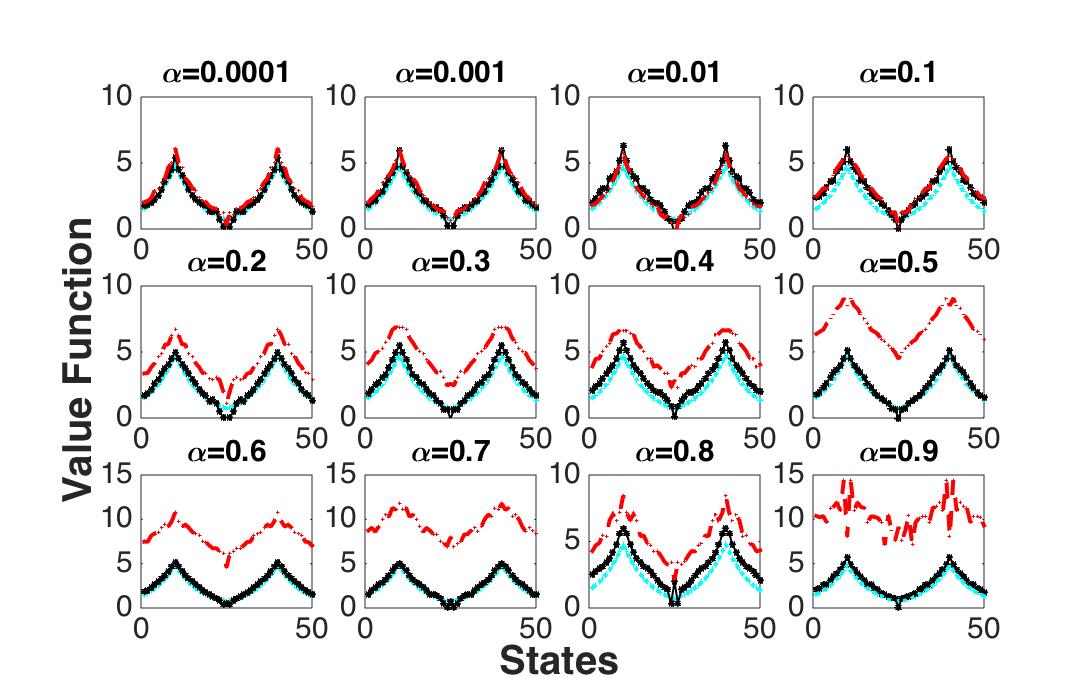}
\caption{Chain Domain}
\label{fig:chain}
\end{figure}


\subsection{Energy Management Domain}
In this experiment, we compare the performance of the algorithms on an energy management domain~\cite{liu2015uai}.
The decision maker must decide how much energy to purchase or sell subject to stochastic prices. This problem is relevant in the context of utilities as well as in settings such as hybrid vehicles. The prices are generated by a Markov chain process. The amount of available storage is limited and degrades with use. The degradation process is based on the physical properties of lithium-ion batteries and discourages fully charging or discharging the battery. The energy arbitrage problem is closely related to the broad class of inventory management problems, with the storage level corresponding to the inventory. However, there are no known results describing the structure of optimal threshold policies in energy storage.

Note that since this is an off-policy evaluation problem, the formulated $A\theta=b$ does not have a solution, and thus the optimal MSPBE($\theta^*$) (resp. MSBE($\theta^*$)) does not reduce to $0$. 
 The result is averaged over $200$ runs,  and $\alpha = 0.001$ for both GTD2 and GTD2-MP is chosen via comparison studies for each algorithm.
As can be seen from Figure~\ref{fig:inv}, GTD2-MP performs much better than  GTD2 in the transient state. Then after reaching the steady state, as can be seen from Table~\ref{tab:inv}, we can see that GTD2-MP reaches a better steady-state solution than the GTD algorithm.
From the steady-state reported in Table~\ref{tab:inv}, we can see that GTD-MP and GTD2-MP usually reach a far better final solution than TD and GTD/GTD2 algorithms.


\begin{figure}
\centering
\begin{minipage}{1\textwidth}
\centering
\includegraphics[width= 1\textwidth, height=2.5in]{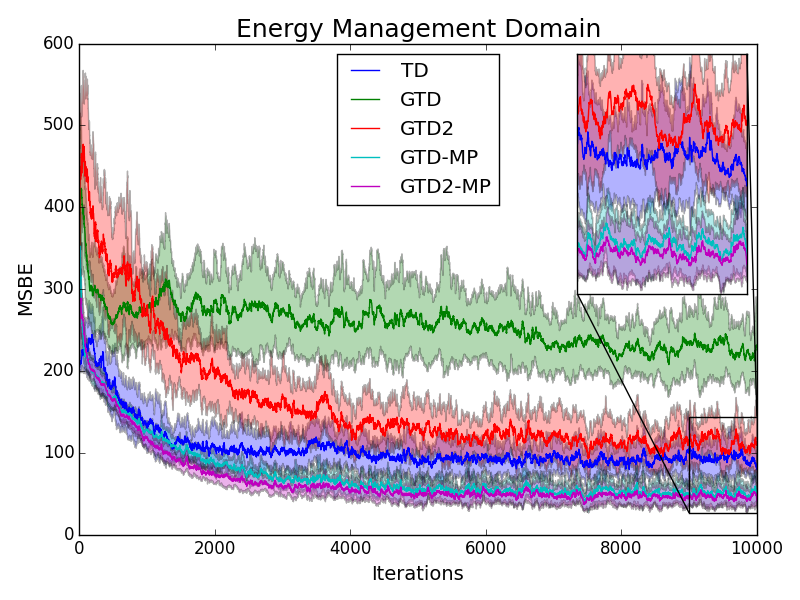}\\
\includegraphics[width= 1\textwidth, height=2.5in]{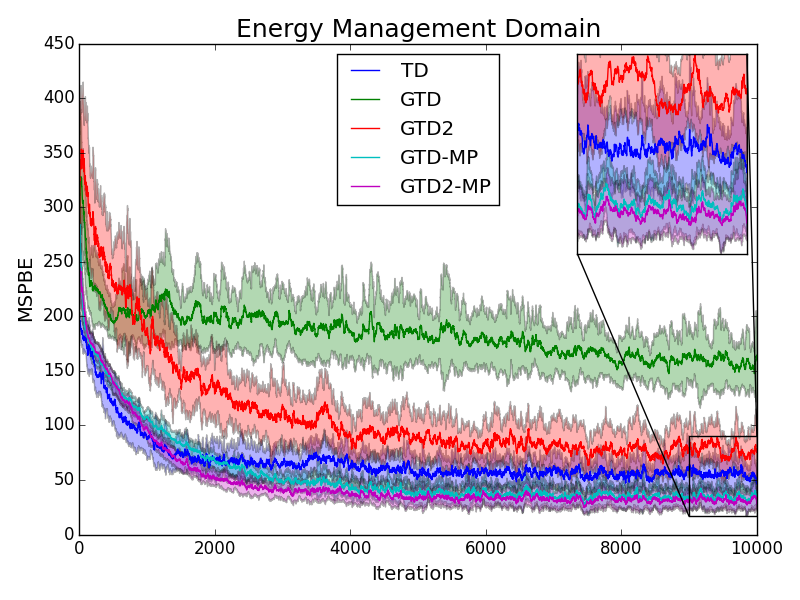}
\end{minipage}
\caption{Energy Management Example}
\label{fig:inv}
\end{figure}

\begin{table}
\centering
\begin{tabular}{|c|c|c|}
\hline 
Algorithm & MSPBE & MSBE        \\
\hline 
TD & $46.743$ & $80.050$      \\
\hline 
GTD & $164.378$ & $231.569$      \\
\hline
GTD2 & $77.139$ & $111.19$      \\
\hline 
GTD-MP & $30.170$ & $44.627$      \\
\hline 
GTD2-MP & $\textbf{27.891}$ & $\textbf{41.028}$    \\
\hline
\end{tabular}
\caption{Steady State Performance Comparison of Battery Management Domain}
\label{tab:inv}
\end{table}


\subsection{Bicycle Balancing and Riding Task}

The bicycle balancing and riding domain~\cite{bicycle:randlov} is a complicated domain. The goal is to learn to balance and ride a bicycle to a target position from the starting location.

To make a fair comparison, the parameter settings are identical to the parameter settings in the work of~\citeauthor{lagoudakis:jmlr}~\citeyear{lagoudakis:jmlr}. The samples are generated via the random walk, {after which, we compare the value function approximation results of TD, GTD2, and GTD2-MP algorithm}.  
From Figure~\ref{fig:bike1}, we can see that both {the GTD2 and GTD2-MP algorithms} reach a much better learning curve than the TD algorithm with significantly reduced variance. Besides, the {GTD2 and GTD2-MP algorithms reach better steady-state solutions} than the TD algorithm, as shown in Table~\ref{tab:bicycle}.

\begin{figure}
\centering
\begin{minipage}{1\textwidth}
\centering
\includegraphics[width= 1\textwidth, height=2.5in]{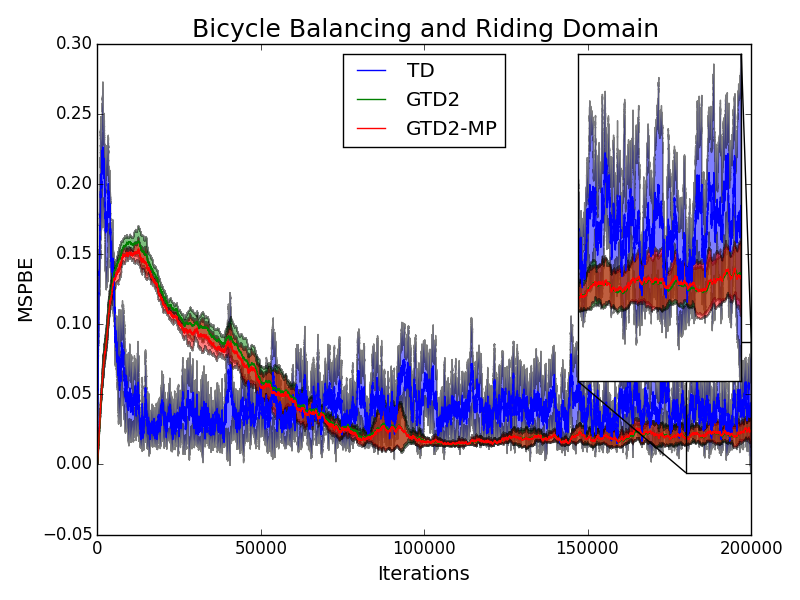}\\
\includegraphics[width= 1\textwidth, height=2.5in]{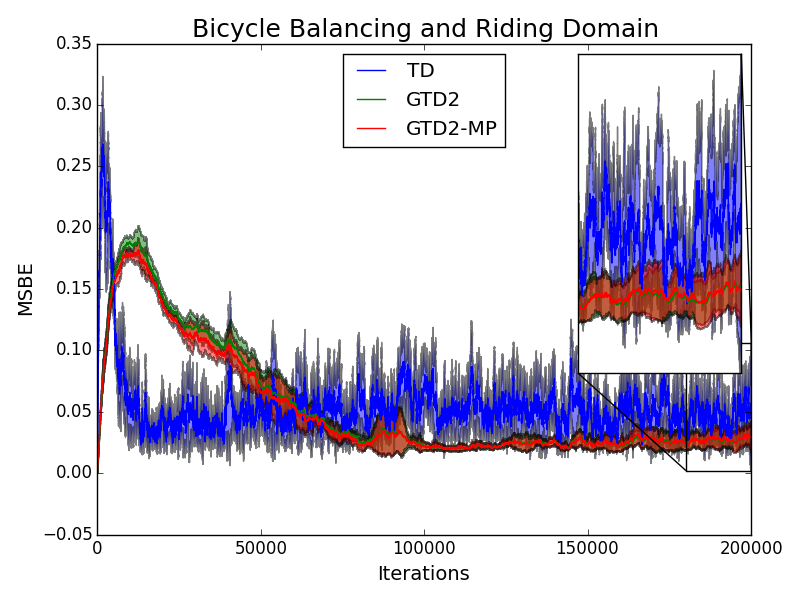}
\end{minipage}
\caption{Energy Management Example}
\label{fig:bike1}
\end{figure}

\begin{table}
\centering
\begin{tabular}{|c|c|c|}
\hline 
Algorithm & MSPBE & MSBE        \\
\hline 
TD & $0.0423$ & $0.0547$      \\
\hline 
GTD2 & $0.0244$ & $0.0300$     \\
\hline 
GTD2-MP & $\textbf{0.0238}$ & $\textbf{0.0297}$    \\
\hline 
\end{tabular}
\caption{Steady State Performance Comparison of Bicycle Domain}
\label{tab:bicycle}
\end{table}

\subsection{Comparison with Other First-Order Policy Evaluation Algorithms}

Here we give an experimental comparison between the gradient-based TD algorithms and the TD algorithm. Based on the experimental results shown above, we make the following empirical conclusions:

\begin{figure}[!ht]
\begin{tabular}{|p{15cm}|}
\hline
\begin{itemize}
\item Of all the gradient-based algorithms, GTD2-MP is the clear winner.
\item For small and medium scale problems, TD is an ideal choice as it converges faster at the initial stage. On the other hand, GTD2-MP  often reaches a better steady-state solution given more number of iterations.
\item For large-scale problems, GTD2-MP is the clear winner over the TD method with both reduced variance and better final solution, as shown in the bicycle and energy management domain.
\item There exist some domains where the $T$ operator is not differentiable, and thus only TD-based algorithms can be applied, such as the optimal stopping problem in~\cite{FPKF:roy}.
\end{itemize}
\\
\hline
\end{tabular}
\caption{Summary of Comparisons between TD and GTD algorithm family}
\label{fig:comp-summary}
\end{figure}

\subsection{Energy Management Domain (Revisited): Control Learning}
Here, we compare the TD, TDC, and GTD-MP variants of GQ-Learning on the battery management domain~\cite{liu2015uai}.  Using the same domain settings as in previous work~\cite{liu2015uai,liu2016ijcai}, we train the three methods on a uniformly random behavior policy for $7,000$ iterations.  We then evaluate the policy, $\theta_t$, learned by each algorithm every $100$ time steps by computing the total reward accumulated by following that policy for $10,000$ iterations, averaged over $10$ runs (see Figure \ref{fig:control_r}).
\begin{figure}
\centering
\begin{minipage}{1\textwidth}
\centering
\includegraphics[width= 1\textwidth, height=2.5in]{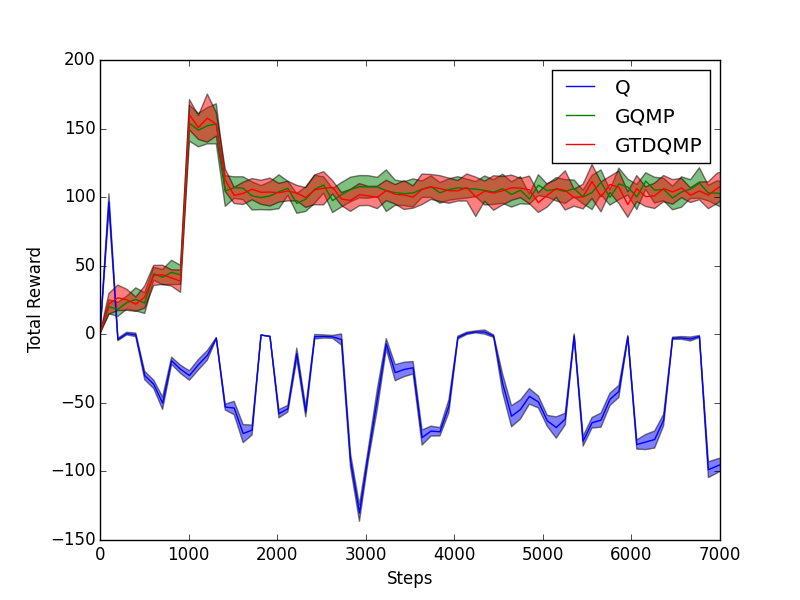}
\end{minipage}
\caption{Average total accumulated reward for learned policies at step $t$.  Shaded regions around mean total reward denote 1 standard deviation.}
\label{fig:control_r}
\end{figure}

All methods were run with $\theta_0$ initialized to the zeros vector.  The TD variant was run with a step size of .0001 to avoid divergence of MSBE while the TDC and GTD-MP variants remained stable in terms of MSBE with a step size of $0.001$ (see Figure \ref{fig:control_msbe}).
\begin{figure}
\centering
\begin{minipage}{1\textwidth}
\centering
\includegraphics[width= 1\textwidth, height=2.5in]{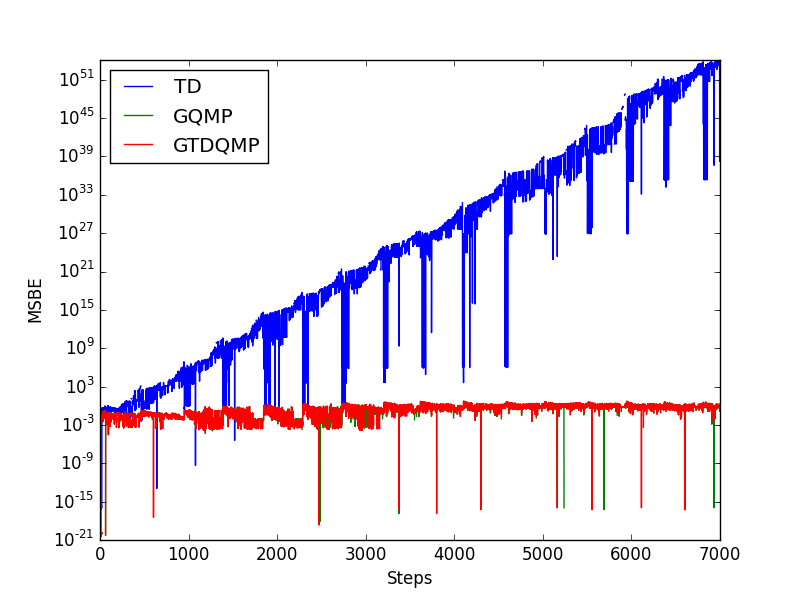}
\end{minipage}
\caption{MSBE diverges for TD variant with step size of $0.001$.  MSBE is plotted on a log scale.}
\label{fig:control_msbe}
\end{figure}


\section{Summary}

In this paper, we showed how gradient TD methods can be shown to be true stochastic gradient methods with respect to a saddle-point primal-dual objective function, which paved the way for the finite-sample analysis of off-policy convergent gradient-based temporal difference learning algorithms such as GTD and GTD2. 
 Both error bound and performance bound are provided, which shows that the value function approximation bound of the GTD algorithms family is $O\left( {\frac{d}{{{n^{1/4}}}}} \right)$. 
 Furthermore, two revised algorithms, namely the projected GTD2 algorithm and the accelerated GTD2-MP algorithm,  are proposed.

There are many interesting directions for future research. Our framework can be easily used to design regularized sparse gradient off-policy TD methods. 

There are several promising future research directions with our proposed proximal gradient TD learning framework. The first promising direction is to explore other compound operator splitting techniques other than primal-dual splitting. As we have shown in previous chapters, new algorithms can be designed if there exist methods that can split the operator so that the product of expectations can be avoided, and this operator splitting formulation does not have to be the primal-dual formulation. We have explored two primal-dual formulations, one is based on the convex conjugate function, and the other is based on dual norm representation. It would be interesting to see if there are any other compound operator splitting techniques that will lead to a family of new algorithms along with possibly faster convergence rate.

Another interesting future direction is to explore proximal gradient TD algorithms with transfer RL. Given multiple different but related tasks, knowledge transfer is desirable and will help faster learning, less sample complexity, and better generalization ability. There are various types of transfer learning at different levels, such as instance-level transfer, feature-level transfer, and parameter-level transfer. As we know, from a transfer learning perspective, off-policy learning is instance level transfer learning. It would be interesting to see if other transfer RL problems can be formulated as saddle-point problems and if there is a similar finite-sample analysis as well.
Another interesting direction is to design new objective functions for TD learning. Since Bellman error is an expectation function (of the TD error), both MSPBE and NEU are (weighted) \textit{norm of expectations} of the TD error. This is where the biased sampling problem comes from. It would be desirable if a new set of objectives can be designed, in which the biased sampling problem can be avoided. How to combine model-free temporal difference learning and learning with a generative model~\cite{chen:2018:scalable} is another interesting question to explore.

\section*{Acknowledgements}
This material is based upon work supported by the  National Science Foundation under Grant Nos. IIS-1216467 and ETRI funds at Auburn University. Ji Liu is in part supported by NSF CCF1718513, IBM faculty award, and NEC fellowship. 
Any opinions, findings, and conclusions or recommendations expressed in this material are those of the authors and do not necessarily reflect the views of the NSF.

\appendix


\section{Proof of Lemma~\ref{lem:abbound}}
\begin{proof}
From the boundedness of the features (by $L$) and the rewards (by ${{\rm{R}}_{\max }}$), we have
\begin{align}
||A||_2 &= ||\mathbb{E}[{\rho _t}{\phi _t}\Delta \phi _t^ \top ]|{|_2} \\
&\le {{\max }_s}||\rho (s)\phi (s){{(\Delta \phi (s))}^ \top }|{|_2} \\
&\le {\rho _{\max }}{{\max }_s}||\phi (s)|{|_2}{{\max }_s}||\phi (s) - \gamma \phi '(s)|{|_2} \\
&\le {\rho _{\max }}{{\max }_s}||\phi (s)|{|_2}{{\max }_s}\left( {||\phi (s)|{|_2} + \gamma ||\phi '(s)|{|_2}} \right) \\
&\le (1 + \gamma ){\rho _{\max }}L^2 d.
\end{align}
The second inequality is obtained by the consistent inequality of matrix norm, the third inequality comes from the triangular norm inequality, and the fourth inequality comes from the vector norm inequality $||\phi (s)|{|_2} \le ||\phi (s)|{|_\infty }\sqrt d  \le L\sqrt d$. The bound on $||b||_2$ can be derived in a similar way as follows.
\begin{align}
||b||_2 &= ||\mathbb{E}[{\rho _t}{\phi _t}r_t ]|{|_2} \\
&\le {\max _s}||\rho (s)\phi (s)r(s)|{|_2} \\
&\le {\rho _{\max }}{{\max }_s}||\phi (s)|{|_2}{{\max }_s}||r (s)|{|_2} \\
&\le {\rho _{\max }}L{{\rm{R}}_{\max }}.
\end{align}
It completes the proof.
\end{proof}

\section{Proof of Proposition~\ref{pro:hp}}

\begin{proof}
The proof of Proposition~\ref{pro:hp} mainly relies on Proposition~3.2 in the work of~\citeauthor{RobustSA:2009}~\citeyear{RobustSA:2009}. We just need to map our convex-concave {\em stochastic} saddle-point problem~\eqref{eq:sp}, i.e., 
\begin{equation}
\mathop{\min}\limits_{\theta\in\Theta}\mathop{\max}\limits_{y\in Y}\left({L(\theta,y)=\left\langle {b-A\theta,y}\right\rangle -\frac{1}{2}||y||_{M}^{2}}\right)
\end{equation}
to the one in Section~3 of that work and show that it satisfies all the conditions necessary for Proposition~3.2. Assumption~\ref{ass:xyfeasible} guarantees that our feasible sets $\Theta$ and $Y$ satisfy the necessary conditions, as they are non-empty bounded closed convex subsets of $\mathbb{R}^d$. We also see that our objective function $L(\theta,y)$ is {\em convex} in $\theta\in\Theta$ and {\em concave} in $y\in Y$, and also {\em Lipschitz continuous} on $\Theta\times Y$. It is known that in the above setting, our saddle-point problem~\eqref{eq:sp} is solvable, i.e.,~the corresponding {\em primal} and {\em dual} optimization problems: $\min_{\theta\in\Theta}\big[\max_{y\in Y}L(\theta,y)\big]$ and $\max_{y\in Y}\big[\min_{\theta\in\Theta}L(\theta,y)\big]$ are solvable with equal optimal values, denoted $L^*$, and pairs $(\theta^*,y^*)$ of optimal solutions to the respective problems from the set of saddle-points of $L(\theta,y)$ on $\Theta\times Y$. 

For our problem, the {\em stochastic sub-gradient vector} $G$ is defined as

\begin{equation}
G(\theta ,y) = \left[ {\begin{array}{*{20}{c}}
{{G_\theta }(\theta ,y)}\\
{ - {G_y}(\theta ,y)}
\end{array}} \right] 
= \left[ {\begin{array}{*{20}{c}}
{ - \hat{A}_t^ \top y}\\
{ - (\hat{b}_t - \hat{A}_t\theta  - \hat{M}_ty)}
\end{array}} \right]. 
\end{equation}
%
This guarantees that the {\em deterministic sub-gradient vector}

\begin{equation}
g(\theta ,y) = \left[ {\begin{array}{*{20}{c}}
{{g_\theta }(\theta ,y)}\\
{ - {g_y}(\theta ,y)}
\end{array}} \right] 
= \left[ {\begin{array}{*{20}{c}}
\mathbb{E}\big[G_\theta(\theta,y)\big] \\
- \mathbb{E}\big[G_y(\theta,y)\big]
\end{array}} \right] 
\end{equation}
%
%
is well-defined, i.e.,~$g_\theta(\theta,y)\in\partial_\theta L(\theta,y)$ and $g_y(\theta,y)\in\partial_y L(\theta,y)$.

We also consider the Euclidean stochastic approximation (E-SA) setting in which the {\em distance generating functions} $\omega_\theta:\Theta\rightarrow\mathbb{R}$ and $\omega_y:Y\rightarrow\mathbb{R}$ are simply defined as
\begin{equation}
\omega_\theta=\frac{1}{2}||\theta||_2^2, \quad\quad \omega_y=\frac{1}{2}||y||_2^2,
\end{equation}
modulus $1$ w.r.t.~$||\cdot||_2$, and thus, $\Theta^o=\Theta$ and $Y^o=Y$ (see pp.~1581~and~1582~in~\citeR{RobustSA:2009}). This allows us to equip the set $Z=\Theta\times Y$ with the distance generating function 
\begin{equation}
\omega(z)=\frac{\omega_\theta(\theta)}{2D_\theta^2}+\frac{\omega_y(y)}{2D_y^2},
\end{equation}
where $D_\theta$ and $D_y$ are defined in Assumption~\ref{ass:xyfeasible}. 

Now that we consider the Euclidean case and set the norms to $\ell_2$-norm, we can compute upper-bounds on the expectation of the dual norm of the stochastic sub-gradients 
\begin{equation}
\mathbb{E}\left[||G_\theta(\theta,y)||^2_{*,\theta}\right] \le M_{*,\theta}^2, \quad \mathbb{E}\left[||G_y(\theta,y)||^2_{*,y}\right] \le M_{*,y}^2,
\end{equation}
where $||\cdot||_{*,\theta}$ and $||\cdot||_{*,y}$ are the dual norms in $\Theta$ and $Y$, respectively. Since we are in the Euclidean setting and use the $\ell_2$-norm, the dual norms are also $\ell_2$-norm, and thus, to compute $M_{*,\theta}$, we need to upper-bound $\mathbb{E}\left[||G_\theta(\theta,y)||_2^2\right]$ and $\mathbb{E}\left[||G_y(\theta,y)||_2^2\right]$. 

To bound these two quantities, we use the following equality that holds for any random variable $x$: 
\begin{equation}
\mathbb{E}[||x||_2^2] =  \mathbb{E}[||x - {\mu _x}||_2^2] + ||{\mu _x}||_2^2,
\end{equation}
where $\mu_x = \mathbb{E}[x]$. Here is how we bound $\mathbb{E}\left[||G_\theta(\theta,y)||_2^2\right]$,
\begin{align}
\nonumber
\mathbb{E}\left[||G_\theta(\theta,y)||_2^2\right] 
&= \mathbb{E}[||\hat A_t^ \top y|{|}^2_2] \\
\nonumber
&= \mathbb{E}[||\hat A_t^ \top y - {A^ \top }y|{|}_2^2] + ||{A^ \top }y|{|}^2_2 \\
\nonumber
&\le {\sigma }^2_2 + {(||A|{|_2}||y|{|_2})^2}\\
&\le {\sigma }^2_2 + ||A|{|}^2_2{R^2},
\end{align}
\noindent where the first inequality is from the definition of $\sigma_2$ in Eq.~\eqref{eq:sigma123} and the consistent inequality of the matrix norm, and the second inequality comes from the boundedness of the feasible sets in Assumption~\ref{ass:xyfeasible}. Similarly we bound $\mathbb{E}\left[||G_y(\theta,y)||_2^2\right]$ as follows:
\begin{align}
\mathbb{E}[||G_y(\theta,y&)||_2^2] = \mathbb{E}[||{{\hat b}_t} - {{\hat A}_t}\theta  - {{\hat M}_t}y||^2_2] \\
&= ||b - A\theta  + My||_2^2 \\
\nonumber
&{  }   + \mathbb{E}[||{{\hat b}_t} - {{\hat A}_t}\theta  - {{\hat M}_t}y - (b - A\theta  - My)||_2^2]\\
&\le {(||b|{|_2} + ||A|{|_2}||\theta |{|_2} + \tau ||y|{|_2})^2} + \sigma _1^2\\
&\le \big(||b||_2 + (||A||_2 + \tau)R\big)^2 + \sigma _1^2,
\end{align}
%
\noindent where these inequalities come from the definition of $\sigma_1$ in Eq.~\eqref{eq:sigma123} and the boundedness of the feasible sets in Assumption~\ref{ass:xyfeasible}. This means that in our case we can compute $M_{*,\theta }^2, M_{*,y}^2$ as
\begin{align}
M_{*,\theta }^2 &= \sigma _2^2 + ||A||_2^2{R^2},\\
M_{*,y}^2 &= \big(||b||_2 + (||A||_2 + \tau)R\big)^2 + \sigma _1^2,
\end{align}
and as a result
\begin{align}
M_*^2 &= 2D_\theta ^2M_{*,\theta }^2 + 2D_y^2M_{*,y}^2 = 2R^2(M_{*,\theta }^2 + M_{*,y}^2) \\
&= R^2\left(\sigma^2 + ||A||_2^2R^2 + \big(||b||_2 + (||A||_2 + \tau)R\big)^2\right)\\
 & \le {\left( {{R^2}\left( {2||A|{|_2} + \tau } \right) + R(\sigma  + ||b|{|_2})} \right)^2},
\end{align}
where the inequality comes from the fact that $\forall a,b,c \ge 0,a^2 + b^2 + c^2 \le (a + b +c)^2$. Thus, we may write $M_*$ as
\begin{equation}
\label{eq:Mstar-def}
M_* = {{R^2}\left( {2||A|{|_2} + \tau } \right) + R(\sigma  + ||b|{|_2})}.
\end{equation}
Now we have all the pieces ready to apply Proposition 3.2 in the work of~\citeauthor{RobustSA:2009}~\citeyear{RobustSA:2009} and obtain a high-probability bound on ${\rm{Err}}({\bar \theta_n},{\bar y_n})$, where $\bar \theta_n$ and $\bar y_n$ (see Eq.~\eqref{eq:bartheta}) are the outputs of the revised GTD algorithm in Algorithm~\ref{alg:pgtd2}. From Proposition 3.2, if we set the step-size in Algorithm~\ref{alg:pgtd2} (our revised GTD algorithm) to $\alpha_t=\frac{2c}{M_*\sqrt{5n}}$, where $c>0$ is a positive constant, $M_*$ is defined by Eq.~\eqref{eq:Mstar-def}, and $n$ is the number of training samples in $\mathcal{D}$, with probability of at least $1-\delta$, we have
\begin{equation}
\label{eq:hp2}
{\rm{Err}}({\bar \theta _n},{\bar y_n}) \le \sqrt {\frac{5}{n}} (8 + 2\log \frac{2}{\delta }){R^2}\left( {2 ||A|{|_2} + \tau + \frac{{||b|{|_2} + \sigma }}{R}} \right).
\end{equation}
%
%
Note that we obtain Eq.~\eqref{eq:hp2} by setting $c=1$ and the ``light-tail'' assumption in Eq.~\eqref{eq:lt}  guarantees that we satisfy the condition in Eq.~(3.16) in the work of~\citeauthor{RobustSA:2009}~\citeyear{RobustSA:2009}, which is necessary for the high-probability bound in their Proposition~3.2~to hold. The proof is complete by replacing $||A||_2$ and $||b||_2$ from Lemma~\ref{lem:abbound}.
\end{proof}

\section{Proof of Proposition~\ref{pro:4} }

\begin{proof}
From Lemma~\ref{lem:v}, we have
\begin{align}
V - {{\bar v}_n} =& \;{(I - \gamma \Pi P)^{ - 1}} \times\\
&\;\big[ {\left( {V - \Pi V} \right) + \Phi {{C}^{ - 1}}(b-A{{\bar{\theta} }_n})} \big].
\end{align}
Applying $\ell_2$-norm w.r.t.~the distribution $\xi$ to both sides of this equation, we obtain
\begin{align}
\label{eq:ccc1}
||V - \bar v_n||_\xi \le &||(I - \gamma \Pi P)^{-1}||_\xi \times \\
&\big(||V - \Pi V||_\xi + ||\Phi C^{-1}(b - A\bar\theta_n)||_\xi\big).  \nonumber
\end{align}
Since $P$ is the kernel matrix of the target policy $\pi$ and $\Pi$ is the orthogonal projection w.r.t.~$\xi$, the stationary distribution of $\pi$, we may write 
\begin{equation}
||{(I - \gamma \Pi P)^{ - 1}}|{|_{\xi} } \le \frac{1}{{1 - \gamma }}.
\end{equation}
Moreover, we may upper-bound the term $||\Phi C^{-1}(b - A\bar\theta_n)||_\xi$ in Eq.~\eqref{eq:ccc1} using the following inequalities: 
\begin{align}
||\Phi {C^{ - 1}}(b &- A{{\bar \theta }_n})|{|_\xi } \\ 
&\le ||\Phi {C^{ - 1}}(b - A{{\bar \theta }_n})|{|_2}\sqrt {{\xi _{\max }}} \\ 
&\le ||\Phi |{|_2}||{C^{ - 1}}|{|_2}||(b - A{{\bar \theta }_n})|{|_{{M^{ - 1}}}}\sqrt {\tau {\xi _{\max }}} \\
&\le (L\sqrt d )(\frac{1}{\nu })\sqrt {2{\rm{Err}}({{\bar \theta }_n},{{\bar y}_n})} \sqrt {\tau {\xi _{\max }}} \\
&= {\frac{L}{\nu }\sqrt {2d\tau {\xi _{\max }}{\rm{Err}}({{\bar \theta }_n},{{\bar y}_n})} },
\end{align}
where the third inequality is the result of upper-bounding $||(b - A\bar \theta_n)||_M^{-1}$ using Eq.~\eqref{eq:err3} and the fact that $\nu  = 1/||{C^{ - 1}}||{^2_2} = 1/{\lambda _{\max }}({C^{ - 1}}) = {\lambda _{\min }}(C)$ ($\nu$ is the smallest eigenvalue of the covariance matrix $C$).
\end{proof}


\section{Proof of Proposition~\ref{pro:5} }

\begin{proof}
Using the triangle inequality, we may write 
\begin{equation}
||V-{\bar v_n}|||_\xi \le || {\bar v_n} - \Phi {\theta ^*}|{|_{\xi} } + ||V-\Phi {\theta ^*}||_\xi. 
\label{eq:errdecomp}
\end{equation} 
The second term on the right-hand side of Eq.~\eqref{eq:errdecomp} can be upper-bounded by Lemma~\ref{lem:kolter}. Now we upper-bound the first term as follows: 
%
%
\begin{align}
|| {{\bar v}_n} &- \Phi \theta^*||_{\xi} ^2\\
 &  = ||\Phi {\bar \theta _n} - \Phi {\theta ^*}||_{\xi} ^2 \\
 &  =  ||{{\bar \theta }_n} - {\theta ^*}||_C^2\\
 & \le ||{{\bar \theta }_n} - {\theta ^*}||_{{A^ \top }{M^{ - 1}}A}^2||{({A^ \top }{M^{ - 1}}A)^{ - 1}}|{|_2}||C|{|_2}\\
 & = ||A({{\bar \theta }_n} - {\theta ^*})||_{{M^{ - 1}}}^2||{({A^ \top }{M^{ - 1}}A)^{ - 1}}|{|_2}||C|{|_2}\\
 &  = ||A{{\bar \theta }_n} - b||_{{M^{ - 1}}}^2\frac{{{\tau _C}}}{{{\sigma _{\min }}({A^ \top }{M^{ - 1}}A)}},
\end{align}
where ${\tau _C} = {\sigma _{\max }}(C)$ is the largest singular value of $C$, and ${\sigma _{\min }}({A^ \top }{M^{ - 1}}A)$ is the smallest singular value of ${{A^ \top }{M^{ - 1}}A}$. Using the result of Theorem~\ref{thm:1}, with probability at least $1-\delta$, we have
\begin{align}
\frac{1}{2}||A{\bar \theta _n} &- b||_{M^{-1}}^2 \le  \tau {\xi _{\max }}{\rm{Err}}({\bar \theta _n},{\bar y_n}).
\label{eq:thm1variant}
\end{align}
Thus,
\begin{align}
|| {\bar v_n} - \Phi \theta^*||_{\xi} ^2  
& \le \frac{{2{\tau _C}\tau {\xi _{\max }}}}{{{\sigma _{\min }}({A^ \top }{M^{ - 1}}A)}}{\rm{Err}}({\bar \theta _n},{\bar y_n})
\label{eq:barvnstar}
\end{align}
From Eqs.~\eqref{eq:errdecomp},~\eqref{eq:kolter2}, and~\eqref{eq:barvnstar}, the result of Eq.~\eqref{eq:pro5} can be derived, which completes the proof.
\end{proof}

\section{Battery Domain}

The problem represents an energy arbitrage model with multiple finite \emph{known} price levels and a stochastic evolution given a limited storage capacity. In particular, the storage is assumed to be an electrical battery that degrades when energy is stored or retrieved.  Energy prices are governed by a Markov process with states $\Theta$. There are two energy prices in each time step: $p^i : \Theta \rightarrow \mathbb{R}$ is the purchase (or input) price and $p^o: \Theta \rightarrow \mathbb{R}$ is the sell (or output) price. The parameter $\theta$ vary between 0 and 10 and their evolution is governed by a martingale with a normal distribution around the mean.

We use $s$ to denote the available battery capacity with $s_0$ denoting the initial capacity. The current state of charge is denoted by $x$ or $y$ and must satisfy that $0 \le x_t \le s_t$ at any time step $t$. The action is the amount of energy to charge or discharge, which is denoted by $u$. Positive $u$ indicates that energy is purchased to charge the battery; negative $u$ indicates the sale of energy. 

The battery storage degrades with use. The degradation is a function of the battery capacity when charged or discharged. We use a general model of battery degradation with a specific focus on Li-ion batteries. The degradation function $d(x,u) \in \mathbb{R}$ represent the battery capacity loss after starting at the state of charge $x \ge 0$ and charging (discharging if negative) by $u$ with $-x \le u \le s_0$.
This function indicates the loss of capacity, such that:
\[ s_{t+1} = s_t - d(x_t,u_t) \]

The state set in the Markov decision problem is composed of $(x,s,\theta)$ where $x$ is the state of charge, $s$ is the battery capacity, and $\theta\in\Theta$ is the state of the price process. The available actions in a state $(x,s,\theta)$ are $u$ such that $-x \le u \le s-x$. The transition is from $(x_t,s_t,\theta_t)$ to $(x_{t+1},s_{t+1},\theta_{t+1})$ given action $u_t$ is:
\begin{align}
x_{t+1} &= x_t + u_t \\
s_{t+1} &= s_t - d(x_t,u_t)
\end{align}
The probability of this transition is given by $P[\theta_{t+1} \vert \theta_t]$. The reward for this transition is:
\[ r((x_t,s_t,\theta_t), u_t) = \begin{cases} 
        - u_t \cdot p^i - c^d \cdot d(x_t,u_t) &\text{if } u_t \ge 0 \\ 
        - u_t \cdot p^o - c^d \cdot d(x_t,u_t) &\text{if } u_t < 0 \end{cases}.\]
That is, the reward captures the monetary value of the transaction minus a penalty for degradation of the battery. Here, $c^d$ represents the cost of a unit of lost battery capacity.

The Bellman optimality equations for this problem are:
\begin{equation}
\label{eq:bellman_optimality_simple} 
\begin{aligned}
q_T(x,s,\theta) &= 0 \\
v_t(x,s,\theta_t) &= \min \bigl\{p^i_{\theta_t} \pos{u} + p^o_{\theta_t} \negt{u}  +  \\
        & \quad + c^d \, d(x,u) +  \\
        &\quad +q_t(x + u,s- d(x,u),\theta_t) :  \\
        &\qquad :  u \in [ -x, s - x ] \bigr\}   \\
q_t(x , s, \theta_t) &= \gamma \cdot \operatorname{E}[v_{t+1}( x,s,\theta_{t+1}) ]
\end{aligned}
\end{equation}
where the expectation $\operatorname{E}[v_{t+1}( x,s,\theta_{t+1}) ]$ is taken over $P(\theta_{t+1} | \theta_t)$. 

The value function is approximated using piece-wise linear features of three types $\phi^1$, $\phi^2$, $\phi^3$ defined as a function of the MDP state as follows:
\begin{align}
\phi^1_{w,q} (x,s,\theta) &= \begin{cases} \pos{x - w} &\text{if } \theta = q \\ 0 &\text{otherwise} \end{cases} \\
\phi^2_{w,q} (x,s,\theta) &= \begin{cases} \pos{s - w} &\text{if } \theta = q \\ 0 &\text{otherwise} \end{cases} \\
\phi^3_{w,q} (x,s,\theta) &= \begin{cases} \pos{s + x - w} &\text{if } \theta = q \\ 0 &\text{otherwise} \end{cases}
\end{align}
Here, $w \in \{0,0.1,\ldots,0.9,1\}$ and $q\in\Theta$.

These features can be conveniently used to approximate a piece-wise linear function.

\bibliography{thesisbib}

\begin{thebibliography}{}

\bibitem[\protect\BCAY{Antos, Szepesv{\'a}ri,\ \BBA\ Munos}{Antos
  et~al.}{2008}]{antos08learning}
Antos, A., Szepesv{\'a}ri, C., \BBA\ Munos, R. \BBOP2008\BBCP.
\newblock \BBOQ Learning near-optimal policies with bellman-residual
  minimization based fitted policy iteration and a single sample path\BBCQ\
\newblock {\Bem Machine Learning}, {\Bem 71\/}(1), 89--129.

\bibitem[\protect\BCAY{Baird}{Baird}{1995}]{Baird:ResidualAlgorithms1995}
Baird, L.~C. \BBOP1995\BBCP.
\newblock \BBOQ Residual algorithms: Reinforcement learning with function
  approximation\BBCQ\
\newblock In {\Bem International Conference on Machine Learning}, \BPGS\
  30--37.

\bibitem[\protect\BCAY{Bauschke\ \BBA\ Combettes}{Bauschke\ \BBA\
  Combettes}{2011}]{BOOK2011PROXSPLIT}
Bauschke, H.~H.\BBACOMMA\  \BBA\ Combettes, P.~L. \BBOP2011\BBCP.
\newblock {\Bem Convex analysis and monotone operator theory in Hilbert
  spaces}.
\newblock Springer.

\bibitem[\protect\BCAY{Bertsekas\ \BBA\ Tsitsiklis}{Bertsekas\ \BBA\
  Tsitsiklis}{1996}]{ndp:book}
Bertsekas, D.\BBACOMMA\  \BBA\ Tsitsiklis, J. \BBOP1996\BBCP.
\newblock {\Bem Neuro-Dynamic Programming}.
\newblock Athena Scientific, Belmont, Massachusetts.

\bibitem[\protect\BCAY{Bhatnagar, Sutton, Ghavamzadeh,\ \BBA\ Lee}{Bhatnagar
  et~al.}{2009}]{Bhatnagar09NA}
Bhatnagar, S., Sutton, R., Ghavamzadeh, M., \BBA\ Lee, M. \BBOP2009\BBCP.
\newblock \BBOQ Natural actor-critic algorithms\BBCQ\
\newblock {\Bem Automatica}, {\Bem 45\/}(11), 2471--2482.

\bibitem[\protect\BCAY{Borkar}{Borkar}{2008}]{borkar:book}
Borkar, V. \BBOP2008\BBCP.
\newblock {\Bem Stochastic Approximation: A Dynamical Systems Viewpoint}.
\newblock Cambridge University Press.

\bibitem[\protect\BCAY{Boyd\ \BBA\ Vandenberghe}{Boyd\ \BBA\
  Vandenberghe}{2004}]{boyd}
Boyd, S.\BBACOMMA\  \BBA\ Vandenberghe, L. \BBOP2004\BBCP.
\newblock {\Bem Convex Optimization}.
\newblock Cambridge University Press.

\bibitem[\protect\BCAY{Bradtke\ \BBA\ Barto}{Bradtke\ \BBA\
  Barto}{1996}]{bradtke-barto:LSTD}
Bradtke, S.\BBACOMMA\  \BBA\ Barto, A. \BBOP1996\BBCP.
\newblock \BBOQ Linear least-squares algorithms for temporal difference
  learning\BBCQ\
\newblock {\Bem Machine Learning}, {\Bem 22}, 33--57.

\bibitem[\protect\BCAY{Bubeck}{Bubeck}{2014}]{bubeck2014optml}
Bubeck, S. \BBOP2014\BBCP.
\newblock \BBOQ Theory of convex optimization for machine learning\BBCQ\
\newblock In {\Bem arXiv:1405.4980}.

\bibitem[\protect\BCAY{Chen, Lan,\ \BBA\ Ouyang}{Chen
  et~al.}{2013}]{chen2013optimal}
Chen, Y., Lan, G., \BBA\ Ouyang, Y. \BBOP2013\BBCP.
\newblock \BBOQ Optimal primal-dual methods for a class of saddle point
  problems\BBCQ\
\newblock In {\Bem arXiv:1309.5548}.

\bibitem[\protect\BCAY{Chen, Li,\ \BBA\ Wang}{Chen
  et~al.}{2018}]{chen:2018:scalable}
Chen, Y., Li, L., \BBA\ Wang, M. \BBOP2018\BBCP.
\newblock \BBOQ Scalable bilinear $\pi$ learning using state and action
  features\BBCQ\
\newblock In {\Bem arXiv preprint arXiv:1804.10328}.

\bibitem[\protect\BCAY{Choi\ \BBA\ Van~Roy}{Choi\ \BBA\
  Van~Roy}{2006}]{FPKF:roy}
Choi, D.\BBACOMMA\  \BBA\ Van~Roy, B. \BBOP2006\BBCP.
\newblock \BBOQ A generalized kalman filter for fixed point approximation and
  efficient temporal-difference learning\BBCQ\
\newblock {\Bem Discrete Event Dynamic Systems}, {\Bem 16\/}(2), 207--239.

\bibitem[\protect\BCAY{Dalal, Sz{\"o}r{\'e}nyi, Thoppe,\ \BBA\ Mannor}{Dalal
  et~al.}{2018a}]{td-finite:dalal:2018}
Dalal, G., Sz{\"o}r{\'e}nyi, B., Thoppe, G., \BBA\ Mannor, S. \BBOP2018a\BBCP.
\newblock \BBOQ Finite sample analyses for td (0) with function
  approximation\BBCQ\
\newblock In {\Bem Proceedings of the National Conference on Artificial
  Intelligence (AAAI)}.

\bibitem[\protect\BCAY{Dalal, Thoppe, Szorenyi,\ \BBA\ Mannor}{Dalal
  et~al.}{2018b}]{td-finite:colt:dalal:2018}
Dalal, G., Thoppe, G., Szorenyi, B., \BBA\ Mannor, S. \BBOP2018b\BBCP.
\newblock \BBOQ Finite sample analysis of two-timescale stochastic
  approximation with applications to reinforcement learning\BBCQ\
\newblock In {\Bem Proceedings of the 31st Conference On Learning Theory},
  \BPGS\ 1199--1233.

\bibitem[\protect\BCAY{Dann, Neumann,\ \BBA\ Peters}{Dann
  et~al.}{2014}]{dann2014tdsurvey}
Dann, C., Neumann, G., \BBA\ Peters, J. \BBOP2014\BBCP.
\newblock \BBOQ Policy evaluation with temporal differences: A survey and
  comparison\BBCQ\
\newblock {\Bem Journal of Machine Learning Research}, {\Bem 15}, 809--883.

\bibitem[\protect\BCAY{Devolder}{Devolder}{2011}]{inexact:stochastic:devolder2011}
Devolder, O. \BBOP2011\BBCP.
\newblock \BBOQ Stochastic first order methods in smooth convex
  optimization\BBCQ\
\newblock \BTR, Universit{\'e} catholique de Louvain, Center for Operations
  Research and Econometrics.

\bibitem[\protect\BCAY{Du, Chen, Li, Xiao,\ \BBA\ Zhou}{Du
  et~al.}{2017}]{gtd:du:2017stochastic}
Du, S.~S., Chen, J., Li, L., Xiao, L., \BBA\ Zhou, D. \BBOP2017\BBCP.
\newblock \BBOQ Stochastic variance reduction methods for policy
  evaluation\BBCQ\
\newblock In {\Bem arXiv preprint arXiv:1702.07944}.

\bibitem[\protect\BCAY{Duchi, Agarwal, Johansson,\ \BBA\ Jordan}{Duchi
  et~al.}{2012}]{duchi2012ergodic}
Duchi, J., Agarwal, A., Johansson, M., \BBA\ Jordan, M. \BBOP2012\BBCP.
\newblock \BBOQ Ergodic mirror descent\BBCQ\
\newblock {\Bem SIAM Journal on Optimization}, {\Bem 22\/}(4), 1549--1578.

\bibitem[\protect\BCAY{Geist, Scherrer, Lazaric,\ \BBA\ Ghavamzadeh}{Geist
  et~al.}{2012}]{DantzigRL:2012}
Geist, M., Scherrer, B., Lazaric, A., \BBA\ Ghavamzadeh, M. \BBOP2012\BBCP.
\newblock \BBOQ {A Dantzig Selector Approach to Temporal Difference
  Learning}\BBCQ\
\newblock In {\Bem International Conference on Machine Learning}, \BPGS\
  1399--1406.

\bibitem[\protect\BCAY{Geist\ \BBA\ Scherrer}{Geist\ \BBA\
  Scherrer}{2014}]{geist2014off:trace:jmlr}
Geist, M.\BBACOMMA\  \BBA\ Scherrer, B. \BBOP2014\BBCP.
\newblock \BBOQ Off-policy learning with eligibility traces: a survey\BBCQ\
\newblock {\Bem The Journal of Machine Learning Research}, {\Bem 15\/}(1),
  289--333.

\bibitem[\protect\BCAY{Ghavamzadeh, Lazaric, Maillard,\ \BBA\
  Munos}{Ghavamzadeh et~al.}{2010}]{lstdrp:nips2010}
Ghavamzadeh, M., Lazaric, A., Maillard, O., \BBA\ Munos, R. \BBOP2010\BBCP.
\newblock \BBOQ {LSTD with Random Projections}\BBCQ\
\newblock In {\Bem Proceedings of the International Conference on Neural
  Information Processing Systems}, \BPGS\ 721--729.

\bibitem[\protect\BCAY{Ghavamzadeh, Lazaric, Munos,\ \BBA\ Hoffman}{Ghavamzadeh
  et~al.}{2011}]{LASSOTD:2011}
Ghavamzadeh, M., Lazaric, A., Munos, R., \BBA\ Hoffman, M. \BBOP2011\BBCP.
\newblock \BBOQ {Finite-Sample Analysis of Lasso-TD}\BBCQ\
\newblock In {\Bem Proceedings of the 28th International Conference on Machine
  Learning}, \BPGS\ 1177--1184.

\bibitem[\protect\BCAY{Juditsky\ \BBA\ Nemirovski}{Juditsky\ \BBA\
  Nemirovski}{2011}]{sra2011optimization}
Juditsky, A.\BBACOMMA\  \BBA\ Nemirovski, A. \BBOP2011\BBCP.
\newblock {\Bem Optimization for Machine Learning}.
\newblock MIT Press.

\bibitem[\protect\BCAY{Juditsky, Nemirovskii,\ \BBA\ Tauvel}{Juditsky
  et~al.}{2008}]{juditsky2008solving}
Juditsky, A., Nemirovskii, A., \BBA\ Tauvel, C. \BBOP2008\BBCP.
\newblock \BBOQ Solving variational inequalities with stochastic mirror-prox
  algorithm\BBCQ\
\newblock In {\Bem arXiv:0809.0815}.

\bibitem[\protect\BCAY{Kakade\ \BBA\ Langford}{Kakade\ \BBA\
  Langford}{2002}]{Kakade02AO}
Kakade, S.\BBACOMMA\  \BBA\ Langford, J. \BBOP2002\BBCP.
\newblock \BBOQ Approximately optimal approximate reinforcement learning\BBCQ\
\newblock In {\Bem Proceedings of the Nineteenth International Conference on
  Machine Learning}, \BPGS\ 267--274.

\bibitem[\protect\BCAY{Kamal}{Kamal}{2010}]{td-finite:kamal:2010}
Kamal, S. \BBOP2010\BBCP.
\newblock \BBOQ On the convergence, lock-in probability, and sample complexity
  of stochastic approximation\BBCQ\
\newblock {\Bem SIAM Journal on Control and Optimization}, {\Bem 48\/}(8),
  5178--5192.

\bibitem[\protect\BCAY{Kolter}{Kolter}{2011}]{Kolter:offpolicyTD}
Kolter, Z. \BBOP2011\BBCP.
\newblock \BBOQ {The Fixed Points of Off-Policy TD}\BBCQ\
\newblock In {\Bem Advances in Neural Information Processing Systems 24},
  \BPGS\ 2169--2177.

\bibitem[\protect\BCAY{Lagoudakis\ \BBA\ Parr}{Lagoudakis\ \BBA\
  Parr}{2003}]{lagoudakis:jmlr}
Lagoudakis, M.\BBACOMMA\  \BBA\ Parr, R. \BBOP2003\BBCP.
\newblock \BBOQ Least-squares policy iteration\BBCQ\
\newblock {\Bem Journal of Machine Learning Research}, {\Bem 4}, 1107--1149.

\bibitem[\protect\BCAY{Lakshminarayanan\ \BBA\
  Szepesv{\'a}ri}{Lakshminarayanan\ \BBA\
  Szepesv{\'a}ri}{2018}]{td-finite:Csaba:2018}
Lakshminarayanan, C.\BBACOMMA\  \BBA\ Szepesv{\'a}ri, C. \BBOP2018\BBCP.
\newblock \BBOQ Linear stochastic approximation: How far does constant
  step-size and iterate averaging go?\BBCQ\
\newblock In {\Bem International Conference on Artificial Intelligence and
  Statistics}, \BPGS\ 1347--1355.

\bibitem[\protect\BCAY{Lazaric, Ghavamzadeh,\ \BBA\ Munos}{Lazaric
  et~al.}{2010a}]{Lazaric10AC}
Lazaric, A., Ghavamzadeh, M., \BBA\ Munos, R. \BBOP2010a\BBCP.
\newblock \BBOQ Analysis of a classification-based policy iteration
  algorithm\BBCQ\
\newblock In {\Bem Proceedings of the Twenty-Seventh International Conference
  on Machine Learning}, \BPGS\ 607--614.

\bibitem[\protect\BCAY{Lazaric, Ghavamzadeh,\ \BBA\ Munos}{Lazaric
  et~al.}{2010b}]{Lazaric_finite-sampleanalysis}
Lazaric, A., Ghavamzadeh, M., \BBA\ Munos, R. \BBOP2010b\BBCP.
\newblock \BBOQ {Finite-Sample Analysis of LSTD}\BBCQ\
\newblock In {\Bem Proceedings of 27th International Conference on Machine
  Learning}, \BPGS\ 615--622.

\bibitem[\protect\BCAY{Lazaric, Ghavamzadeh,\ \BBA\ Munos}{Lazaric
  et~al.}{2012}]{Lazaric12FS}
Lazaric, A., Ghavamzadeh, M., \BBA\ Munos, R. \BBOP2012\BBCP.
\newblock \BBOQ Finite-sample analysis of least-squares policy iteration\BBCQ\
\newblock {\Bem Journal of Machine Learning Research}, {\Bem 13}, 3041--3074.

\bibitem[\protect\BCAY{Liu, Liu, Ghavamzadeh, Mahadevan,\ \BBA\ Petrik}{Liu
  et~al.}{2015}]{liu2015uai}
Liu, B., Liu, J., Ghavamzadeh, M., Mahadevan, S., \BBA\ Petrik, M.
  \BBOP2015\BBCP.
\newblock \BBOQ Finite-sample analysis of proximal gradient td
  algorithms.\BBCQ\
\newblock In {\Bem UAI}, \BPGS\ 504--513.

\bibitem[\protect\BCAY{Liu, Liu, Ghavamzadeh, Mahadevan,\ \BBA\ Petrik}{Liu
  et~al.}{2016}]{liu2016ijcai}
Liu, B., Liu, J., Ghavamzadeh, M., Mahadevan, S., \BBA\ Petrik, M.
  \BBOP2016\BBCP.
\newblock \BBOQ Proximal gradient temporal difference learning
  algorithms.\BBCQ\
\newblock In {\Bem IJCAI}, \BPGS\ 4195--4199.

\bibitem[\protect\BCAY{Liu, Mahadevan,\ \BBA\ Liu}{Liu
  et~al.}{2012}]{ROTD:NIPS2012}
Liu, B., Mahadevan, S., \BBA\ Liu, J. \BBOP2012\BBCP.
\newblock \BBOQ Regularized off-policy {TD}-learning\BBCQ\
\newblock In {\Bem Advances in Neural Information Processing Systems 25},
  \BPGS\ 845--853.

\bibitem[\protect\BCAY{Liu, Xie, Xu, Ghavamzadeh, Chow, Lyu,\ \BBA\ Yoon}{Liu
  et~al.}{2018}]{liubo:nips:2018}
Liu, B., Xie, T., Xu, Y., Ghavamzadeh, M., Chow, Y., Lyu, D., \BBA\ Yoon, D.
  \BBOP2018\BBCP.
\newblock \BBOQ A block coordinate ascent algorithm for mean-variance
  optimization\BBCQ\
\newblock In {\Bem Advances in Neural Information Processing Systems}.

\bibitem[\protect\BCAY{Maei}{Maei}{2011}]{maei2011thesis}
Maei, H. \BBOP2011\BBCP.
\newblock {\Bem Gradient temporal-difference learning algorithms}.
\newblock Ph.D.\ thesis, University of Alberta.

\bibitem[\protect\BCAY{Maei\ \BBA\ Sutton}{Maei\ \BBA\
  Sutton}{2010}]{gq:maei2010}
Maei, H.\BBACOMMA\  \BBA\ Sutton, R. \BBOP2010\BBCP.
\newblock \BBOQ {GQ ($\lambda$): A general gradient algorithm for
  temporal-difference prediction learning with eligibility traces}\BBCQ\
\newblock In {\Bem Proceedings of the Third Conference on Artificial General
  Intelligence}, \BPGS\ 91--96.

\bibitem[\protect\BCAY{Mahadevan\ \BBA\ Liu}{Mahadevan\ \BBA\
  Liu}{2012}]{mahadevan:MID:2012}
Mahadevan, S.\BBACOMMA\  \BBA\ Liu, B. \BBOP2012\BBCP.
\newblock \BBOQ {Sparse Q-learning with Mirror Descent}\BBCQ\
\newblock In {\Bem Proceedings of the Conference on Uncertainty in AI}.

\bibitem[\protect\BCAY{Mahadevan, Liu, Thomas, Dabney, Giguere, Jacek, Gemp,\
  \BBA\ Liu}{Mahadevan et~al.}{2014}]{proximalrl}
Mahadevan, S., Liu, B., Thomas, P., Dabney, W., Giguere, S., Jacek, N., Gemp,
  I., \BBA\ Liu, J. \BBOP2014\BBCP.
\newblock \BBOQ Proximal reinforcement learning: A new theory of sequential
  decision making in primal-dual spaces.\BBCQ\
\newblock In {\Bem arXiv:1405.6757}.

\bibitem[\protect\BCAY{Munos\ \BBA\ Szepesv{\'a}ri}{Munos\ \BBA\
  Szepesv{\'a}ri}{2008}]{Munos08FT}
Munos, R.\BBACOMMA\  \BBA\ Szepesv{\'a}ri, C. \BBOP2008\BBCP.
\newblock \BBOQ Finite time bounds for fitted value iteration\BBCQ\
\newblock {\Bem Journal of Machine Learning Research}, {\Bem 9}, 815--857.

\bibitem[\protect\BCAY{Nemirovski, Juditsky, Lan,\ \BBA\ Shapiro}{Nemirovski
  et~al.}{2009}]{RobustSA:2009}
Nemirovski, A., Juditsky, A., Lan, G., \BBA\ Shapiro, A. \BBOP2009\BBCP.
\newblock \BBOQ Robust stochastic approximation approach to stochastic
  programming\BBCQ\
\newblock {\Bem SIAM Journal on Optimization}, {\Bem 19}, 1574--1609.

\bibitem[\protect\BCAY{Palaniappan\ \BBA\ Bach}{Palaniappan\ \BBA\
  Bach}{2016}]{saddlepoint:palaniappan:2016stochastic}
Palaniappan, B.\BBACOMMA\  \BBA\ Bach, F. \BBOP2016\BBCP.
\newblock \BBOQ Stochastic variance reduction methods for saddle-point
  problems\BBCQ\
\newblock In {\Bem Advances in Neural Information Processing Systems}, \BPGS\
  1416--1424.

\bibitem[\protect\BCAY{Parr, Painter-Wakefield, Li,\ \BBA\ Littman}{Parr
  et~al.}{2007}]{bebf:parr:icml07}
Parr, R., Painter-Wakefield, C., Li, L., \BBA\ Littman, M. \BBOP2007\BBCP.
\newblock \BBOQ Analyzing feature generation for value function
  approximation\BBCQ\
\newblock In {\Bem Proceedings of the International Conference on Machine
  Learning}, \BPGS\ 737--744.

\bibitem[\protect\BCAY{Pires\ \BBA\ Szepesv{\'a}ri}{Pires\ \BBA\
  Szepesv{\'a}ri}{2012}]{pires:2012:inverse}
Pires, B.~A.\BBACOMMA\  \BBA\ Szepesv{\'a}ri, C. \BBOP2012\BBCP.
\newblock \BBOQ Statistical linear estimation with penalized estimators: an
  application to reinforcement learning\BBCQ\
\newblock In {\Bem Proceedings of the 29th International Conference on Machine
  Learning}, \BPGS\ 1535--1542.

\bibitem[\protect\BCAY{Polyak\ \BBA\ Juditsky}{Polyak\ \BBA\
  Juditsky}{1992}]{polyak1992acceleration}
Polyak, B.\BBACOMMA\  \BBA\ Juditsky, A. \BBOP1992\BBCP.
\newblock \BBOQ Acceleration of stochastic approximation by averaging\BBCQ\
\newblock {\Bem SIAM Journal on Control and Optimization}, {\Bem 30\/}(4),
  838--855.

\bibitem[\protect\BCAY{Prashanth, Korda,\ \BBA\ Munos}{Prashanth
  et~al.}{2014}]{drift:prashanth2014fast}
Prashanth, L., Korda, N., \BBA\ Munos, R. \BBOP2014\BBCP.
\newblock \BBOQ Fast {LSTD} using stochastic approximation: Finite time
  analysis and application to traffic control\BBCQ\
\newblock In {\Bem Machine Learning and Knowledge Discovery in Databases},
  \BPGS\ 66--81. Springer.

\bibitem[\protect\BCAY{Qin\ \BBA\ Li}{Qin\ \BBA\ Li}{2014}]{ZHIWEI2014}
Qin, Z.\BBACOMMA\  \BBA\ Li, W. \BBOP2014\BBCP.
\newblock \BBOQ {Sparse Reinforcement Learning via Convex Optimization}\BBCQ\
\newblock In {\Bem Proceedings of the 31st International Conference on Machine
  Learning}.

\bibitem[\protect\BCAY{Randl{\o}v\ \BBA\ Alstr{\o}m}{Randl{\o}v\ \BBA\
  Alstr{\o}m}{1998}]{bicycle:randlov}
Randl{\o}v, J.\BBACOMMA\  \BBA\ Alstr{\o}m, P. \BBOP1998\BBCP.
\newblock \BBOQ Learning to drive a bicycle using reinforcement learning and
  shaping\BBCQ\
\newblock In {\Bem Proceedings of the International Conference on Machine
  Learning}, \lowercase{\BVOL}~98, \BPGS\ 463--471.

\bibitem[\protect\BCAY{Sutton\ \BBA\ Barto}{Sutton\ \BBA\
  Barto}{1998}]{sutton-barto:book}
Sutton, R.\BBACOMMA\  \BBA\ Barto, A.~G. \BBOP1998\BBCP.
\newblock {\Bem {Reinforcement Learning: An Introduction}}.
\newblock MIT Press.

\bibitem[\protect\BCAY{Sutton, Maei, Precup, Bhatnagar, Silver,
  Szepesv{\'a}ri,\ \BBA\ Wiewiora}{Sutton et~al.}{2009}]{tdc:2009}
Sutton, R., Maei, H., Precup, D., Bhatnagar, S., Silver, D., Szepesv{\'a}ri,
  C., \BBA\ Wiewiora, E. \BBOP2009\BBCP.
\newblock \BBOQ Fast gradient-descent methods for temporal-difference learning
  with linear function approximation\BBCQ\
\newblock In {\Bem International Conference on Machine Learning}, \BPGS\
  993--1000.

\bibitem[\protect\BCAY{Sutton, Szepesv{\'a}ri,\ \BBA\ Maei}{Sutton
  et~al.}{2008}]{Sutton:GTD1:2008}
Sutton, R., Szepesv{\'a}ri, C., \BBA\ Maei, H. \BBOP2008\BBCP.
\newblock \BBOQ A convergent o(n) algorithm for off-policy temporal-difference
  learning with linear function approximation\BBCQ\
\newblock In {\Bem Neural Information Processing Systems}, \BPGS\ 1609--1616.

\bibitem[\protect\BCAY{Szepesv{\'a}ri}{Szepesv{\'a}ri}{2010}]{szepesvari2010algorithms}
Szepesv{\'a}ri, C. \BBOP2010\BBCP.
\newblock \BBOQ Algorithms for reinforcement learning\BBCQ\
\newblock {\Bem Synthesis Lectures on Artificial Intelligence and Machine
  Learning}, {\Bem 4\/}(1), 1--103.

\bibitem[\protect\BCAY{Tagorti\ \BBA\ Scherrer}{Tagorti\ \BBA\
  Scherrer}{2014}]{bruno:lstdlambda}
Tagorti, M.\BBACOMMA\  \BBA\ Scherrer, B. \BBOP2014\BBCP.
\newblock \BBOQ Rate of convergence and error bounds for {LSTD}
  ($\lambda$)\BBCQ\
\newblock In {\Bem arXiv:1405.3229}.

\bibitem[\protect\BCAY{Thoppe\ \BBA\ Borkar}{Thoppe\ \BBA\
  Borkar}{2015}]{td-finite:thoppe:2015}
Thoppe, G.\BBACOMMA\  \BBA\ Borkar, V.~S. \BBOP2015\BBCP.
\newblock \BBOQ A concentration bound for stochastic approximation via
  alekseev's formula\BBCQ\
\newblock In {\Bem arXiv preprint arXiv:1506.08657}.

\bibitem[\protect\BCAY{Wang, Chen, Liu, Ma,\ \BBA\ Liu}{Wang
  et~al.}{2017}]{td-finite:wang:2017}
Wang, Y., Chen, W., Liu, Y., Ma, Z., \BBA\ Liu, T. \BBOP2017\BBCP.
\newblock \BBOQ Finite sample analysis of the gtd policy evaluation algorithms
  in markov setting\BBCQ\
\newblock In {\Bem Advances in Neural Information Processing Systems}, \BPGS\
  5504--5513.

\bibitem[\protect\BCAY{White\ \BBA\ White}{White\ \BBA\
  White}{2016}]{td:comparison:adam:2016}
White, A.\BBACOMMA\  \BBA\ White, M. \BBOP2016\BBCP.
\newblock \BBOQ Investigating practical linear temporal difference
  learning\BBCQ\
\newblock In {\Bem Proceedings of the 2016 International Conference on
  Autonomous Agents and Multiagent Systems}, \BPGS\ 494--502.

\bibitem[\protect\BCAY{Yu}{Yu}{2017}]{gtd:lambda:yu:2017}
Yu, H. \BBOP2017\BBCP.
\newblock \BBOQ On convergence of some gradient-based temporal-differences
  algorithms for off-policy learning\BBCQ\
\newblock In {\Bem arXiv preprint arXiv:1712.09652}.

\bibitem[\protect\BCAY{Yu}{Yu}{2012}]{yu2012error}
Yu, H. \BBOP2012\BBCP.
\newblock \BBOQ Least squares temporal difference methods: An analysis under
  general conditions\BBCQ\
\newblock {\Bem SIAM Journal on Control and Optimization}, {\Bem 50\/}(6),
  3310--3343.

\end{thebibliography}
\bibliographystyle{theapa}
\end{document}